\documentclass[3p,times]{elsarticle}

\usepackage{ecrc}

\volume{00}

\firstpage{1}

\journalname{Procedia Computer Science}

\runauth{}

\jid{procs}

\jnltitlelogo{Procedia Computer Science}

\CopyrightLine{2011}{Published by Elsevier Ltd.}

\usepackage{graphicx}
\usepackage{amsthm,amsmath,amsfonts,amssymb}
\usepackage{mathrsfs}
\usepackage{booktabs}
\usepackage{algorithm} %format of the algorithm 
\usepackage{algorithmic} %format of the algorithm 
\usepackage{subfig}
\usepackage{multirow}
\usepackage{multicol}
\usepackage{xcolor}
\usepackage{diagbox}
\usepackage{hyperref}
\usepackage[document]{ragged2e}

\theoremstyle{plain}
\newtheorem{theorem}{Theorem}[section]

\theoremstyle{definition}

\theoremstyle{remark}

\usepackage[figuresright]{rotating}

\begin{document}

\begin{frontmatter}



\dochead{}
%% e.g. \dochead{17th International Conference on Dynamical Processes in Excited States of Solids}

\title{MPDIoU: A Loss for Efficient and Accurate Bounding Box Regression}

\author[SCUT]{Siliang Ma}
\author[SCUT]{Yong Xu}

\address[SCUT]{Institute of Computer Science and Engineering, South China University of Technology, Guangzhou 510000, China}

\begin{abstract}

  Bounding box regression (BBR) has been widely used in object detection and instance segmentation, which is an important step in object localization. However, most of the existing loss functions for bounding box regression cannot be optimized when the predicted box has the same aspect ratio as the groundtruth box, but the width and height values are exactly different. In order to tackle the issues mentioned above, we fully explore the geometric features of horizontal rectangle and propose a novel bounding box similarity comparison metric $MPDIoU$ based on minimum point distance, which contains all of the relevant factors considered in the existing loss functions, namely overlapping or non-overlapping area, central points distance, and deviation of width and height, while simplifying the calculation process. On this basis, we propose a bounding box regression loss function based on $MPDIoU$, called $\mathcal{L}_{MPDIoU}$. Experimental results show that the MPDIoU loss function is applied to state-of-the-art instance segmentation (e.g., YOLACT) and object detection (e.g.,
  YOLOv7) model trained on PASCAL VOC, MS COCO, and IIIT5k outperforms existing loss functions.
\end{abstract}


\end{frontmatter}

% \linenumbers


\section{Introduction}
\label{sec:intro}

Object detection and instance segmentation are two important problems of computer vision, which have attracted a large scale of researchers' interests during the past few years. Most of the state-of-the-art object detectors (e.g., YOLO series \cite{2017YOLO9000,2018YOLOv3,2020YOLOv4,YOLOv5,YOLOF,2022YOLOv7}, Mask R-CNN \cite{MaskR-CNN}, Dynamic R-CNN \cite{DynamicRCNN} and DETR \cite{DETR}) rely on a bounding box regression (BBR) module to determine the position of objects. Based on this paradigm, a well-designed loss function is of great importance for the success of BBR. So far, most of the existing loss functions for BBR fall into two categories: $\ell_{n}$-norm based loss functions and Intersection over Union (IoU)-based loss functions.

However, most of the existing loss functions for bounding box regression have the same value under different prediction results, which decreases the convergence speed and accuracy of bounding box regression. Therefore, considering the advantages and drawbacks of the existing loss functions for bounding box regression, inspired by the geometric features of horizontal rectangle, we try to design a novel loss function $\mathcal{L}_{MPDIoU}$ based on the minimum points distance for bounding box regression, and use $MPDIoU$ as a new measure to compare the similarity between the predicted bounding box and the groundtruth bounding box in the bounding box regression process. We also provide an easy-implemented solution for calculating $MPDIoU$ between two axis-aligned rectangles, allowing it to be used as an evaluation metric to incorporate $MPDIoU$ into state-of-the-art object detection and instance segmentation algorithms, and we test on some of the mainstream object detection, scene text spotting and instance segmentation datasets such as PASCAL VOC \cite{2015The}, MS COCO \cite{2014Microsoft}, IIIT5k \cite{IIIT5K-Words} and MTHv2 \cite{MTHv2} to verify the performance of our proposed $MPDIoU$.

The contribution of this paper can be summarized as below:\\
1. We considered the advantages and disadvantages of the existing $IoU$-based losses and $\ell_{n}$-norm losses, and then proposed an $IoU$ loss based on minimum points distance called $\mathcal{L}_{MPDIoU}$ to tackle the issues of existing losses and obtain a faster convergence speed and more accurate regression results.\\
2. Extensive experiments have been conducted on object detection, character-level scene text spotting and instance segmentation tasks. Outstanding experimental results validate the superiority of the proposed $MPDIoU$ loss. Detailed ablation studies exhibit the effects of different settings of loss functions and parameter values.

%-------------------------------------------------------------------------
\section{Related Work}

\subsection{Object Detection and Instance Segmentation}
During the past few years, a large number of object detection and instance segmentation methods based on deep learning have been proposed by researchers from different countries and regions. In summary, bounding box regression has been adopted as a basic component in many representative object detection and instance segmentation frameworks \cite{Felzenszwalb2010Object}. In deep models for object detection, R-CNN series \cite{2016Faster}, \cite{2018Cascade}, \cite{2017Mask} adopts two or three bounding box regression modules to obtain higher localization accuracy, while YOLO series \cite{2018YOLOv3,2020YOLOv4,2022YOLOv7} and SSD series \cite{liu2016ssd,fu2017dssd,zhou2018scale} adopt one to achieve faster inference. RepPoints \cite{yang2019reppoints} predicts several points to define a rectangular box. FCOS \cite{tian2019fcos} locates an object by predicting the Euclidean distances from the sampling points to the top, bottom, left and right sides of the groundtruth bounding box. 

As for instance segmentation, PolarMask \cite{xie2020polarmask} predicts the length of n rays from the sampling point to the edge of the object in n directions to segment an instance. There are other detectors, such as RRPN \cite{ma2018arbitrary} and R2CNN \cite{jiang2017r2cnn} adding rotation angle regression to detect arbitrary-orientated objects for remote sensing detection and scene text detection. Mask R-CNN \cite{MaskR-CNN} adds an extra instance mask branch on Faster R-CNN \cite{2016Faster}, while the recent state-of-the-art YOLACT \cite{Bolya_2019_ICCV} does the same thing on RetinaNet \cite{lin2017focal}. To sum up, bounding box regression is one key component of state-of-the-art deep models for object detection and instance segmentation.

\subsection{Scene Text Spotting}
In order to solve the problem of arbitrary shape scene text detection
and recognition, ABCNet \cite{2020ABCNet} and its improved version ABCNet v2 \cite{ABCNetv2} use the BezierAlign to transform the arbitrary-shape texts into regular ones. These methods achieve great progress by using rectification module to unify detection and recognition into end-to-end
trainable systems. \cite{2019Towards} propose RoI Masking to extract the feature
for arbitrarily-shaped text recognition. Similar to \cite{2019Towards,PAN++} try to use
a faster detector for scene text detection. AE TextSpotter \cite{AETextSpotter} uses
the results of recognition to guide detection through language model.
Inspired by \cite{SWINTransformer}, \cite{huang2022swintextspotter} proposed a scene text spotting method based on transformer, which provides instance-level text segmentation results.

\subsection{Loss Function for Bounding Box Regression}
\begin{figure}
  \begin{minipage}[h]{0.45\linewidth}
    \centering
    \subfloat[][]{\includegraphics[scale=0.45]{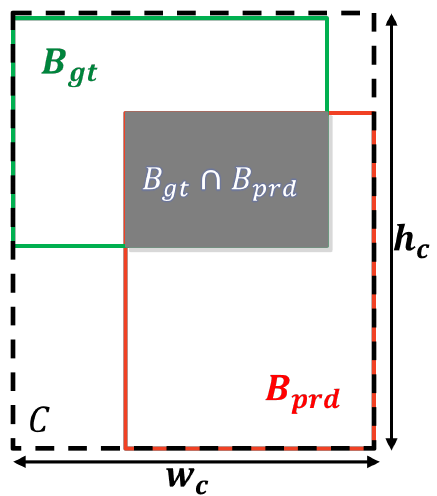}}
    \subfloat[][]{\includegraphics[scale=0.45]{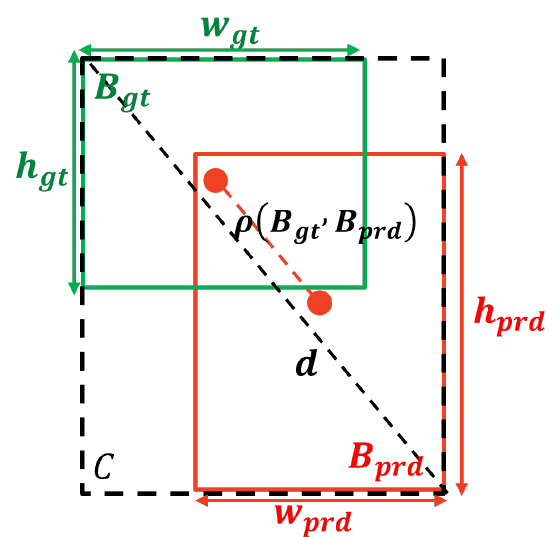}}
    \caption{The calculation factors of the existing metrics for bounding box regression including $GIoU$, $DIoU$, $CIoU$ and $EIoU$.}
    \label{factors}
  \end{minipage}\
  \begin{minipage}[h]{0.48\linewidth}
    \centering
    \subfloat[$\mathcal{L}_{GIoU}=0.75$, $\mathcal{L}_{DIoU}=0.75$, $\mathcal{L}_{CIoU}=0.75$, $\mathcal{L}_{EIoU}=1.25$, \textcolor{red}{$\mathcal{L}_{MPDIoU}=0.79$}]{\includegraphics[scale=0.45]{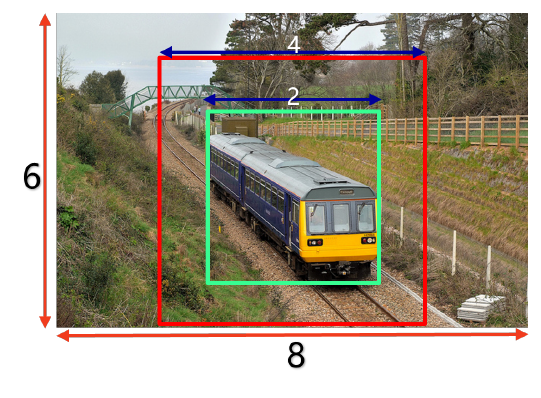}}
    \subfloat[$\mathcal{L}_{GIoU}=0.75$, $\mathcal{L}_{DIoU}=0.75$, $\mathcal{L}_{CIoU}=0.75$, $\mathcal{L}_{EIoU}=1.25$, \textcolor{red}{$\mathcal{L}_{MPDIoU}=0.76$}]{\includegraphics[scale=0.45]{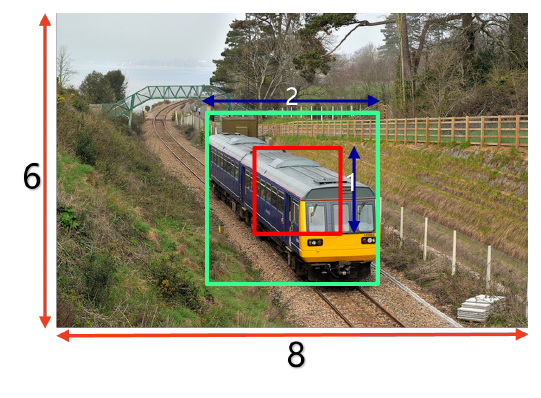}}
    
    \caption{Two cases with different bounding boxes regression results. The \textcolor{green}{green} boxes denote the groundtruth bounding boxes and the \textcolor{red}{red} boxes denote the predicted bounding boxes. The $\mathcal{L}_{GIoU}$, $\mathcal{L}_{DIoU}$, $\mathcal{L}_{CIoU}$, $\mathcal{L}_{EIoU}$ between these two cases are exactly same value, but their $\mathcal{L}_{MPDIoU}$}
    \label{Intro}
  \end{minipage}
\end{figure}

% \begin{figure}
%   \begin{minipage}[t]{width=0.48\linewidth}
%     \centering
%     \includegraphics[width=\linewidth]{factors.pdf}
%     \caption{The calculation factors of the existing metrics for bounding box regression including $GIoU$, $DIoU$, $CIoU$ and $EIoU$.}
%     \label{factors}
%   \end{minipage}
  
%   \begin{minipage}[t]{width=0.48\linewidth}
%     \begin{figure}
%       \subfloat[]{
%       \includegraphics[width=0.24\linewidth]{Intro.pdf}
%       \label{horizontal}
%     }
%     \subfloat[$\mathcal{L}_{GIoU}=0.75$, $\mathcal{L}_{DIoU}=0.75$, $\mathcal{L}_{CIoU}=0.75$, $\mathcal{L}_{EIoU}=1.25$, \textcolor{red}{$\mathcal{L}_{MPDIoU}=0.76$}]{
%       \includegraphics[width=0.24\linewidth]{Intro2.pdf}
%       \label{vertical}
%     }
%     \caption{Two cases with different bounding boxes regression results. The \textcolor{green}{green} boxes denote the groundtruth bounding boxes and the \textcolor{red}{red} boxes denote the predicted bounding boxes. The $\mathcal{L}_{GIoU}$, $\mathcal{L}_{DIoU}$, $\mathcal{L}_{CIoU}$, $\mathcal{L}_{EIoU}$ between these two cases are exactly same value, but their $\mathcal{L}_{MPDIoU}$ values are very different.}
%     \label{Intro}
%     \end{figure}
    
%   \end{minipage}
% \end{figure}

At the very beginning, $\ell_{n}$-norm loss function was widely used for bounding box regression, which was exactly simple but sensitive to various scales. In YOLO v1 \cite{YOLOv1}, square roots for $w$ and $h$ are adopted to mitigate this effect, while YOLO v3 \cite{2018YOLOv3} uses $2-wh$. In order to better calculate the diverse between the groundtruth and the predicted bounding boxes, $IoU$ loss is used since Unitbox \cite{2016UnitBox}. To ensure the training stability, Bounded-$IoU$ loss \cite{tychsen2018improving} introduces the upper bound of $IoU$. 
For training deep models in object detection and instance segmentation, $IoU$-based metrics are suggested to be more consistent than $\ell_{n}$-norm \cite{IoU,tychsen2018improving,2019Generalized}. The original $IoU$ represents the ratio of the intersection area and the union area of the predicted bounding box and the groundtruth bounding box (as Figure \ref{factors}(a) shows), which can be formulated as
\begin{equation}
  IoU=\frac{\mathcal{B}_{gt}\bigcap \mathcal{B}_{prd}}{\mathcal{B}_{gt}\bigcup \mathcal{B}_{prd}},
  \label{LIoU}
\end{equation}
where $\mathcal{B}_{gt}$ denotes the groundtruth bounding box, $\mathcal{B}_{prd}$ denotes the predicted bounding box. As we can see, the original $IoU$ only calculates the union area of two bounding boxes, which can't distinguish the cases that two boxes do not overlap. As equation \ref{LIoU} shows, if $|\mathcal{B}_{gt}\bigcap \mathcal{B}_{prd}|=0$, then $IoU(\mathcal{B}_{gt},\mathcal{B}_{prd})=0$. In this case, $IoU$ can not reflect whether two boxes are in vicinity of each other or very far from each other. Then,  $GIoU$ \cite{2019Generalized} is proposed to tackle this issue. The  $GIoU$ can be formulated as
\begin{equation}
  GIoU=IoU-\frac{\mid \mathcal{C} -\mathcal{B}_{gt}\cup\mathcal{B}_{prd}\mid}{\mid \mathcal{C} \mid},
\end{equation}
where $\mathcal{C}$ is the smallest box covering $\mathcal{B}_{gt}$ and $\mathcal{B}_{prd}$ (as shown in the black dotted box in Figure \ref{factors}(a)), and $\mid C\mid$ is the area of box $\mathcal{C}$. Due to the introduction of the penalty term in  $GIoU$ loss, the predicted box will move toward the target box in nonoverlapping cases.  $GIoU$ loss has been applied to train state-of-the-art object detectors, such as YOLO v3 and Faster R-CNN, and achieves better performance than MSE loss and $IoU$ loss. However,  $GIoU$ will lost effectiveness when the predicted bounding box is absolutely covered by the groundtruth bounding box. In order to deal with this problem, $DIoU$ \cite{zheng2020distance} was proposed with consideration of the centroid points distance between the predicted bounding box and the groundtruth bounding box. The formulation of $DIoU$ can be formulated as 
\begin{equation}
  DIoU=IoU-\frac{\rho ^2 (\mathcal{B}_{gt},\mathcal{B}_{prd})}{\mathcal{C} ^2},
\end{equation}
where $\rho ^2 (\mathcal{B}_{gt},\mathcal{B}_{prd})$ denotes Euclidean distance between the central points of predicted bounding box and groundtruth bounding box (as the red dotted line shown in Figure \ref{factors}(b)). $\mathcal{C} ^2$ denotes the diagonal length of the smallest enclosing rectangle (as the black dotted line shown in Figure \ref{factors}(b)). As we can see, the target of $\mathcal{L}_{DIoU}$ directly minimizes the distance between central points of predicted bounding box and groundtruth bounding box. However, when the central point of predicted bounding box coincides with the central point of groundtruth bounding box, it degrades to the original $IoU$. To address this issue, $CIoU$ was proposed with consideration of both central points distance and the aspect ratio. The formulation of $CIoU$ can be written as follows:
\begin{equation}
 CIoU=IoU-\frac{\rho ^2 (\mathcal{B}_{gt},\mathcal{B}_{prd})}{\mathcal{C} ^2}-\alpha V,
\end{equation}
\begin{equation}
  V =\frac{4}{\pi ^2}(\arctan \frac{w^{gt}}{h^{gt}}-\arctan \frac{w^{prd}}{h^{prd}})^2,
\end{equation}
\begin{eqnarray}
  \alpha =\frac{V}{1-IoU+V}.
\end{eqnarray}
However, the definition of aspect ratio from $CIoU$ is relative value rather than absolute value. To address this issue, $EIoU$ \cite{ZHANG2022146} was proposed based on $DIoU$, which is defined as follows:

\begin{equation}
  EIoU=DIoU-\frac{\rho ^2 (w_{prd},w_{gt})}{(w^{c}) ^2}-\frac{\rho ^2 (h_{prd},h_{gt})}{(h^{c}) ^2}.
\end{equation}
However, as Figure \ref{Intro} shows, the loss functions mentioned above for bounding box regression will lose effectiveness when the predicted bounding box and the groundtruth bounding box have the same aspect ratio with different width and height values, which will limit the convergence speed and accuracy. Therefore, we try to design a novel loss function called $\mathcal{L}_{MPDIoU}$ for bounding box regression with consideration of the advantages included in $\mathcal{L}_{GIoU}$ \cite{2019Generalized}, $\mathcal{L}_{DIoU}$ \cite{zheng2020distance}, $\mathcal{L}_{CIoU}$ \cite{zheng2021enhancing}, $\mathcal{L}_{EIoU}$ \cite{ZHANG2022146}, but also has higher efficiency and accuracy for bounding box regression.

Nonetheless, geometric properties of bounding box regression are actually not fully exploited in existing loss functions. Therefore, we propose $MPDIoU$ loss by minimizing the top-left and bottom-right points distance between the predicted bounding box and the groundtruth bounding box for better training deep models of object detection, character-level scene text spotting and instance segmentation.

\section{Intersection over Union with Minimum Points Distance}

After analyzing the advantages and disadvantages of the $IoU$-based loss functions mentioned above, we start to think how to improve the accuracy and efficiency of bounding box regression. Generally speaking, we use the coordinates of top-left and bottom-right points to define a unique rectangle. Inspired by the geometric properties of bounding boxes, we designed a novel $IoU$-based metric named $MPDIoU$ to minimize the top-left and bottom-right points distance between the predicted bounding box and the groundtruth bounding box directly. The calculation of $MPDIoU$ is summarized in Algorithm \ref{alg1}.
  \begin{algorithm}[H] %算法开始 
    \caption{Intersection over Union with Minimum Points Distance} %算法的题目 
    \label{alg1} %算法的标签 
    \begin{algorithmic}[1] %此处的[1]控制一下算法中的每句前面都有标号 
    \REQUIRE Two arbitrary convex shapes: $ A,B\subseteq \mathbb{S} \in \mathbb{R} ^{n}$, width and height of input image:$w,h$%输入条件(此处的REQUIRE默认关键字为Require，在上面已自定义为Input) 
    \ENSURE $MPDIoU$ %输出结果(此处的ENSURE默认关键字为Ensure在上面已自定义为Output) 
  % if-then-else 
  
  \STATE For $A$ and $B$, $(x_{1}^{A},y_{1}^{A}),(x_{2}^{A},y_{2}^{A})$ denote the top-left and bottom-right point coordinates of $A$, $(x_{1}^{B},y_{1}^{B}),(x_{2}^{B},y_{2}^{B})$ denote the top-left and bottom-right point coordinates of $B$.
  % \STATE For input image, we calculate the diagonal distance $d=\sqrt{(w^2+h^2)}$.
  \STATE $d_{1}^{2}=(x_{1}^{B}-x_{1}^{A})^{2}+(y_{1}^{B}-y_{1}^{A})^{2}$ 
  \STATE $d_{2}^{2}=(x_{2}^{B}-x_{2}^{A})^{2}+(y_{2}^{B}-y_{2}^{A})^{2}$
  \STATE $MPDIoU=\frac{A\bigcap B}{A\bigcup B}-\frac{d_{1}^{2}}{w^2+h^2}-\frac{d_{2}^{2}}{w^2+h^2}$
  \end{algorithmic} 
  \end{algorithm}

% \end{minipage}

% For comparing two specific types of geometric shapes, $C$ can be from the same type. For example, two arbitrary ellipsoids, $C$ could be the smallest ellipsoids enclosing them. Then we calculate the top-left points distance and bottom-right points distance of $A$ and $B$ separately. Finally, we calculate the ratio between the points distances mentioned above and the diagonal distance of the convex area of $A$ and $B$. This represents a normalized measure that focuses on the empty volume (area) between $A$ and $B$. Finally $MPDIoU$ is attained by subtracting this ratio from the $IoU$ value. 

In summary, our proposed $MPDIoU$ simplifies the similarity comparison between two bounding boxes, which can adapt to overlapping or nonoverlapping bounding box regression. Therefore, $MPDIoU$ can be a proper substitute for $IoU$ in all performance measures used in 2D/3D computer vision tasks. In this paper, we only focus on 2D object detection and instance segmentation where we can easily apply $MPDIoU$ as both metric and loss. The extension to non-axis aligned 3D cases is left as future work.
\subsection{MPDIoU as Loss for Bounding Box Regression}
In the training phase, each bounding box $\mathcal{B}_{prd} =[x^{prd},y^{prd},w^{prd},h^{prd}]^T$ predicted by the model is forced to approach its groundtruth box $\mathcal{B}_{gt} = [x^{gt},y^{gt},w^{gt},h^{gt}]^T$ by minimizing loss function below:

\begin{equation}
  \mathcal{L}=\underset{\Theta }{\min } \underset{\mathcal{B} _{gt}\in\mathbb{B}_{gt} }{\sum}\mathcal{L}(\mathcal{B}_{gt},\mathcal{B}_{prd}|\Theta),
\end{equation}
where $\mathbb{B}_{gt}$ is the set of groundtruth boxes, and $\Theta$ is the parameter of deep model for regression. A typical form of $\mathcal{L}$ is $\ell_{n}$-norm, for example, mean-square error (MSE) loss and Smooth-$\ell_{1}$ loss \cite{1992Robust}, which have been widely adopted in object detection \cite{bae2019object}; pedestrian detection \cite{brazil2017illuminating,zhou2019ssa}; scene text spotting \cite{huang2022swintextspotter,lyu2018mask}; 3D object detection \cite{zhou2018voxelnet,shi2019pointrcnn}; pose estimation \cite{sun2019deep,iskakov2019learnable}; and instance segmentation \cite{chen2019tensormask,Bolya_2019_ICCV}. However, recent researches suggest that $\ell_{n}$-norm-based loss functions are not consistent with the evaluation metric, that is, interaction over union (IoU), and instead propose $IoU$-based loss functions \cite{yu2016unitbox,tychsen2018improving,2019Generalized}. Based on the definition of $MPDIoU$ in the previous section, we define the loss function based on $MPDIoU$ as follows:
\begin{equation}
  \mathcal{L}_{MPDIoU}=1-MPDIoU
\end{equation}

As a result, all of the factors of existing loss functions for bounding box regression can be determined by four points coordinates. The conversion formulas are shown as follow:
\begin{equation}\label{C}
  |C|=(max(x_{2}^{gt},x_{2}^{prd})-min(x_{1}^{gt},x_{1}^{prd}))*(max(y_{2}^{gt},y_{2}^{prd})-min(y_{1}^{gt},y_{1}^{prd})),
\end{equation}
\begin{equation}
  x_{c}^{gt}=\frac{x_{1}^{gt}+x_{2}^{gt}}{2},  y_{c}^{gt}=\frac{y_{1}^{gt}+y_{2}^{gt}}{2}, y_{c}^{prd}=\frac{y_{1}^{prd}+y_{2}^{prd}}{2}, x_{c}^{prd}=\frac{x_{1}^{prd}+x_{2}^{prd}}{2},
\end{equation}
\begin{equation}\label{wh}
  w_{gt}=x_{2}^{gt}-x_{1}^{gt}, h_{gt}=y_{2}^{gt}-y_{1}^{gt}, w_{prd}=x_{2}^{prd}-x_{1}^{prd}, h_{prd}=y_{2}^{prd}-y_{1}^{prd}.
\end{equation}

\begin{figure}
  \begin{minipage}[h]{0.48\linewidth}
  \includegraphics[width=\linewidth]{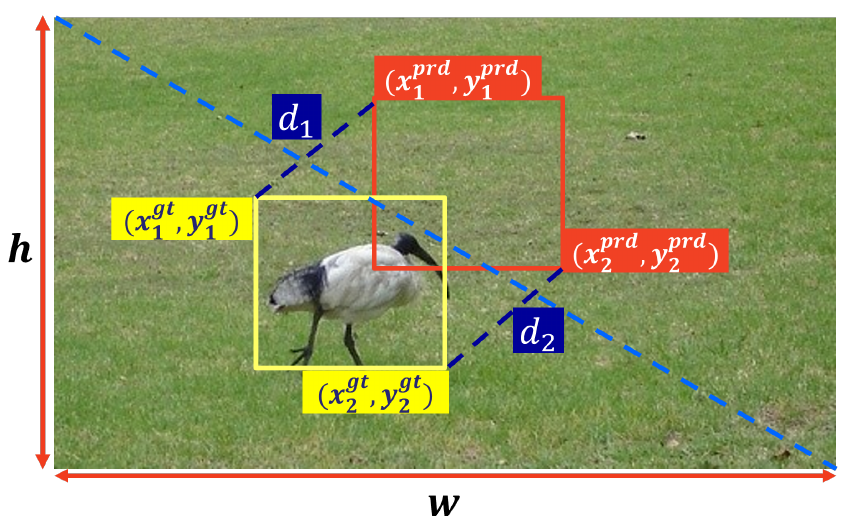}
  \caption{Factors of our proposed $\mathcal{L}_{MPDIoU}$.}
  \label{reg_MPDIoU}
  \end{minipage}
  \begin{minipage}[h]{0.48\linewidth}
    \includegraphics[width=\linewidth]{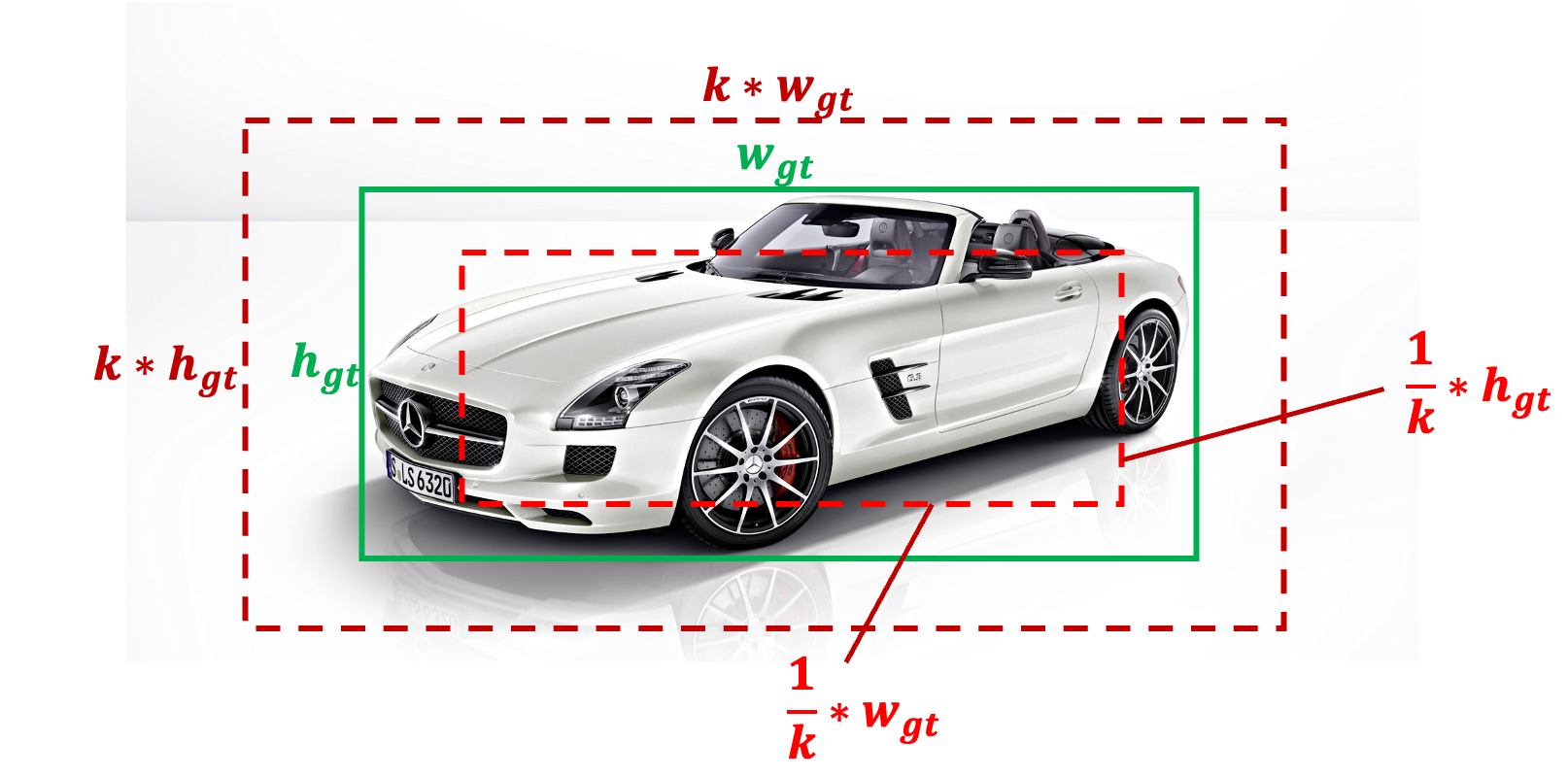}
    \caption{Examples of predicted bounding boxes and groundtruth bounding box with the same aspect ratio but different width and height, where $k>1$ and $k\in R$, the green box denotes the groundtruth box, and the red boxes denote the predicted boxes.}
    \label{Theorm}
    \end{minipage}
\end{figure}

where $|C|$ represents the minimum enclosing rectangle's area covering $\mathcal{B}_{gt}$ and $\mathcal{B}_{prd}$, $(x_c^{gt},y_c^{gt})$ and $(x_c^{prd}, y_c^{prd})$ represent the coordinates of the central points of the groundtruth bounding box and the predicted bounding box, respectively. $w_{gt}$ and $h_{gt}$ represent the width and height of the groundtruth bounding box, $w_{prd}$ and $h_{prd}$ represent the width and height of the predicted bounding box.

From Eq (\ref{C})-(\ref{wh}), we can find that all of the factors considered in the existing loss functions can be determined by the coordinates of the top-left points and the bottom-right points, such as nonoverlapping area, central points distance, deviation of width and height, which means our proposed $\mathcal{L}_{MPDIoU}$ not only considerate, but also simplifies the calculation process.

According to Theorem \ref{Proof}, if the aspect ratio of the predicted bounding boxes and groundtruth bounding box are the same, the predicted bounding box inner the groundtruth bounding box has lower $\mathcal{L}_{MPDIoU}$ value than the prediction box outer the groundtruth bounding box. This characteristic ensures the accuracy of bounding box regression, which tends to provide the predicted bounding boxes with less redudancy.

\begin{theorem}
  We define one groundtruth bounding box as $\mathcal{B}_{gt}$ and two predicted bounding boxes as $\mathcal{B}_{prd1}$ and $\mathcal{B}_{prd2}$. The width and height of the input image are $w$ and $h$, respectively. Assume the top-left and bottom-right coordinates of $\mathcal{B}_{gt}$, $\mathcal{B}_{prd1}$ and $\mathcal{B}_{prd2}$ are $(x_{1}^{gt},y_{1}^{gt},x_{2}^{gt},y_{2}^{gt})$, $(x_{1}^{prd1},y_{1}^{prd1},x_{2}^{prd1},y_{2}^{prd1})$ and $(x_{1}^{prd2},y_{1}^{prd2},x_{2}^{prd2},y_{2}^{prd2})$, then the width and height of $\mathcal{B}_{gt}$, $\mathcal{B}_{prd1}$ and $\mathcal{B}_{prd2}$ can be formulated as ($w_{gt}=y_{2}^{gt}-y_{1}^{gt}$, $h_{gt}=x_{2}^{gt}-x_{1}^{gt}$), ($w_{prd1}=y_{2}^{prd1}-y_{1}^{prd1}$, $h_{prd1}=x_{2}^{prd1}-x_{1}^{prd1}$) and ($w_{prd2}=y_{2}^{prd2}-y_{1}^{prd2}$, $h_{prd2}=x_{2}^{prd2}-x_{1}^{prd2}$). If $w_{prd1}=k*w_{gt}$ and $h_{prd1}=k*h_{gt}$, $w_{prd2}=\frac{1}{k}*w_{gt}$ and $h_{prd2}=\frac{1}{k}*h_{gt}$, where $k>1$ and $k\in N*$\\
  The central points of the $\mathcal{B}_{gt}$, $\mathcal{B}_{prd1}$ and $\mathcal{B}_{prd2}$ are all overlap. Then $GIoU(\mathcal{B}_{gt}, \mathcal{B}_{prd1})=GIoU(\mathcal{B}_{gt}, \mathcal{B}_{prd2})$, $DIoU(\mathcal{B}_{gt}, \mathcal{B}_{prd1})=DIoU(\mathcal{B}_{gt}, \mathcal{B}_{prd2})$, $CIoU(\mathcal{B}_{gt}, \mathcal{B}_{prd1})=CIoU(\mathcal{B}_{gt}, \mathcal{B}_{prd2})$, $EIoU(\mathcal{B}_{gt}, \mathcal{B}_{prd1})=EIoU(\mathcal{B}_{gt}, \mathcal{B}_{prd2})$, but $MPDIoU(\mathcal{B}_{gt}, \mathcal{B}_{prd1})> MPDIoU(\mathcal{B}_{gt}, \mathcal{B}_{prd2})$.
  \label{Proof}
\end{theorem}
\begin{proof}\let\qed\relax
%   $x_{1}^{\mathcal{I}}=max(x_{1}^{prd},x_{1}^{gt})$, $x_{2}^{\mathcal{I}}=min(x_{2}^{prd},x_{2}^{gt})$\\
% $y_{1}^{\mathcal{I}}=max(y_{1}^{prd},y_{1}^{gt}), y_{2}^{\mathcal{I}}*min(y_{2}^{prd},y_{2}^{gt})$

$ \because IoU(\mathcal{B}_{gt}, \mathcal{B}_{prd1}) = \frac{w_{gt}*h_{gt}}{w_{prd1}*h_{prd1}}=\frac{w_{gt}*h_{gt}}{k*w_{gt}*k*h_{gt}}=\frac{1}{k^2},$\\
$ IoU(\mathcal{B}_{gt}, \mathcal{B}_{prd2}) = \frac{w_{prd2}*h_{prd2}}{w_{gt}*h_{gt}}=\frac{\frac{1}{k}*w_{gt}*\frac{1}{k}*h_{gt}}{w_{gt}*h_{gt}}=\frac{1}{k^2}$\\

$\therefore IoU(\mathcal{B}_{gt}, \mathcal{B}_{prd1})=IoU(\mathcal{B}_{gt}, \mathcal{B}_{prd2})$

$ \because$ The central points of the $\mathcal{B}_{gt}$, $\mathcal{B}_{prd1}$ and $\mathcal{B}_{prd2}$ are all overlap.\\

$\therefore GIoU(\mathcal{B}_{gt}, \mathcal{B}_{prd1})=IoU(\mathcal{B}_{gt}, \mathcal{B}_{prd1})=\frac{1}{k^2}$, $GIoU(\mathcal{B}_{gt}, \mathcal{B}_{prd2})=IoU(\mathcal{B}_{gt}, \mathcal{B}_{prd2})=\frac{1}{k^2}$, $DIoU(\mathcal{B}_{gt}, \mathcal{B}_{prd1})=IoU(\mathcal{B}_{gt}, \mathcal{B}_{prd1})=\frac{1}{k^2}$, $DIoU(\mathcal{B}_{gt}, \mathcal{B}_{prd2})=IoU(\mathcal{B}_{gt}, \mathcal{B}_{prd2})=\frac{1}{k^2}$.

$\therefore GIoU(\mathcal{B}_{gt}, \mathcal{B}_{prd1})=GIoU(\mathcal{B}_{gt}, \mathcal{B}_{prd2})$, $DIoU(\mathcal{B}_{gt}, \mathcal{B}_{prd1})=DIoU(\mathcal{B}_{gt}, \mathcal{B}_{prd2})$.

$ \because CIoU(\mathcal{B}_{gt}, \mathcal{B}_{prd1})=IoU(\mathcal{B}_{gt}, \mathcal{B}_{prd1})-\frac{(\frac{4}{\pi ^2}(\arctan \frac{w_{gt}}{h_{gt}}-\arctan \frac{w^{prd1}}{h^{prd1}})^2)^2}{1-IoU(\mathcal{B}_{gt}, \mathcal{B}_{prd1})+\frac{4}{\pi ^2}(\arctan \frac{w_{gt}}{h_{gt}}-\arctan \frac{w^{prd1}}{h^{prd1}})^2}=\frac{1}{k^2}-\frac{(\frac{4}{\pi ^2}(\arctan \frac{w_{gt}}{h_{gt}}-\arctan \frac{k*w_{gt}}{k*h_{gt}})^2)^2}{1-IoU(\mathcal{B}_{gt}, \mathcal{B}_{prd1})+\frac{4}{\pi ^2}(\arctan \frac{w_{gt}}{h_{gt}}-\arctan \frac{k*w_{gt}}{k*h_{gt}})^2}=\frac{1}{k^2}$.\\

$CIoU(\mathcal{B}_{gt}, \mathcal{B}_{prd2})=IoU(\mathcal{B}_{gt}, \mathcal{B}_{prd2})-\frac{(\frac{4}{\pi ^2}(\arctan \frac{w_{gt}}{h_{gt}}-\arctan \frac{w^{prd2}}{h^{prd2}})^2)^2}{1-\frac{1}{k^2}+\frac{4}{\pi ^2}(\arctan \frac{w_{gt}}{h_{gt}}-\arctan \frac{w^{prd2}}{h^{prd2}})^2}=\frac{1}{k^2}-\frac{(\frac{4}{\pi ^2}(\arctan \frac{w_{gt}}{h_{gt}}-\arctan \frac{\frac{1}{k}*w_{gt}}{\frac{1}{k}*h_{gt}})^2)^2}{1-\frac{1}{k^2}+\frac{4}{\pi ^2}(\arctan \frac{w_{gt}}{h_{gt}}-\arctan \frac{\frac{1}{k}*w_{gt}}{\frac{1}{k}*h_{gt}})^2}=\frac{1}{k^2}$.

$\therefore CIoU(\mathcal{B}_{gt}, \mathcal{B}_{prd1})=CIoU(\mathcal{B}_{gt}, \mathcal{B}_{prd2})$.\\
$\because EIoU(\mathcal{B}_{gt}, \mathcal{B}_{prd1})=DIoU(\mathcal{B}_{gt}, \mathcal{B}_{prd1})-\frac{(w_{prd1}-w_{gt})^2}{w_{prd1}^2}-\frac{(h_{prd1}-h_{gt})^2}{h_{prd1}^2}=\frac{1}{k^2}-\frac{(k*w_{gt}-w_{gt})^2}{k^2*w_{gt}^2}-\frac{(k*h_{gt}-h_{gt})^2}{k^2*h_{gt}^2}=\frac{4*k-2*k^2-1}{k^2}$\\

$EIoU(\mathcal{B}_{gt}, \mathcal{B}_{prd2})=DIoU(\mathcal{B}_{gt}, \mathcal{B}_{prd2})-\frac{(w_{gt}-w_{prd2})^2}{w_{gt}^2}-\frac{(h_{gt}-h_{prd2})^2}{h_{gt}^2}=\frac{1}{k^2}-\frac{(w_{gt}-\frac{1}{k}w_{gt})^2}{w_{gt}^2}-\frac{(h_{gt}-\frac{1}{k}h_{gt})^2}{h_{gt}^2}=\frac{4*k-2*k^2-1}{k^2}$.\\

$\therefore EIoU(\mathcal{B}_{gt}, \mathcal{B}_{prd1})=EIoU(\mathcal{B}_{gt}, \mathcal{B}_{prd2})$.
\end{proof}

$\because MPDIoU(\mathcal{B}_{gt}, \mathcal{B}_{prd1})=IoU(\mathcal{B}_{gt}, \mathcal{B}_{prd1})-\frac{(x_{1}^{prd1}-x_{1}^{gt})^2+(y_{1}^{prd1}-y_{1}^{gt})^2+(x_{2}^{prd1}-x_{2}^{gt})^2+(y_{2}^{prd1}-y_{2}^{gt})^2}{w^2+h^2}=\frac{1}{k^2}-\frac{2*((\frac{1}{2}*k*w_{gt}-\frac{1}{2}*w_{gt})^2+(\frac{1}{2}*k*h_{gt}-\frac{1}{2}*h_{gt})^2)}{w^2+h^2}$,\\
$MPDIoU(\mathcal{B}_{gt}, \mathcal{B}_{prd2})=IoU(\mathcal{B}_{gt}, \mathcal{B}_{prd2})-\frac{(x_{1}^{prd2}-x_{1}^{gt})^2+(y_{1}^{prd2}-y_{1}^{gt})^2+(x_{2}^{prd2}-x_{2}^{gt})^2+(y_{2}^{prd2}-y_{2}^{gt})^2}{w^2+h^2}=\frac{1}{k^2}-\frac{2*((\frac{1}{2}*w_{gt}-\frac{1}{2k}*w_{gt})^2+(\frac{1}{2}*h_{gt}-\frac{1}{2k}*h_{gt})^2)}{w^2+h^2}$,\\
$\therefore MPDIoU(\mathcal{B}_{gt}, \mathcal{B}_{prd1})-MPDIoU(\mathcal{B}_{gt}, \mathcal{B}_{prd2})=\frac{1}{4}*(k-1)^2*(w_{gt}^2+h_{gt}^2)-\frac{1}{4}*(1-\frac{1}{k})^2*(w_{gt}^2+h_{gt}^2)=\frac{1}{4}*(w_{gt}^2+h_{gt}^2)*((k-1)^2-(1-\frac{1}{k})^2)$\\
$\because (k-1)^2>(1-\frac{1}{k})^2$\\
$\therefore MPDIoU(\mathcal{B}_{gt}, \mathcal{B}_{prd1})> MPDIoU(\mathcal{B}_{gt}, \mathcal{B}_{prd2})$.
\begin{algorithm}[h]%算法开始 
 
  \begin{algorithmic}[1] %此处的[1]控制一下算法中的每句前面都有标号 
  \REQUIRE Predicted $\mathcal{B}_{prd}$ and ground truth $\mathcal{B}_{gt}$ bounding box coordinates. $\mathcal{B}_{prd}=(x_{1}^{prd},y_{1}^{prd},x_{2}^{prd},y_{2}^{prd})$,$\mathcal{B}_{gt}=(x_{1}^{gt},y_{1}^{gt},x_{2}^{gt},y_{2}^{gt})$, width and height of input image:$w,h$.
  \ENSURE $\mathcal{L}_{IoU}, \mathcal{L}_{MPDIoU}$
  % \mathcal{L}_{DIoU}, \mathcal{L}_{CIoU}, \mathcal{L}_{EIoU},  %输出结果(此处的ENSURE默认关键字为Ensure在上面已自定义为Output) 
% if-then-else 

\STATE For the predicted box $B_{prd}$, ensuring $x_{2}^{prd} > x_{1}^{prd}$ and $y_{2}^{prd} > y_{1}^{prd}$.
% $\hat{x}_{1}^{prd} =min(x_{1}^{prd},x_{2}^{prd})$, $\hat{x}_{2}^{prd} =max(x_{1}^{prd},x_{2}^{prd})$,
% $\hat{y}_{1}^{prd} = min(y_{1}^{prd},y_{2}^{prd})$, $\hat{y}_{2}^{prd} = max(y_{1}^{prd},y_{2}^{prd})$

\STATE $d_{1}^{2}=(x_{1}^{prd}-x_{1}^{gt})^{2}+(y_{1}^{prd}-y_{1}^{gt})^{2}$ 
\STATE $d_{2}^{2}=(x_{2}^{prd}-x_{2}^{gt})^{2}+(y_{2}^{prd}-y_{2}^{gt})^{2}$
\STATE Calculating area of $\mathcal{B}_{gt}$: $A^{gt}=(x_{2}^{gt}-x_{1}^{gt})*(y_{2}^{gt}-y_{1}^{gt})$
\STATE Calculating area of $\mathcal{B}_{prd}$: $A^{prd}=(x_{2}^{prd}-x_{1}^{prd})*(y_{2}^{prd}-y_{1}^{prd})$
\STATE Calculating intersection $\mathcal{I}$ between $\mathcal{B}_{prd}$ and $\mathcal{B}_{gt}$:\\
$x_{1}^{\mathcal{I}}=max(x_{1}^{prd},x_{1}^{gt})$, $x_{2}^{\mathcal{I}}=min(x_{2}^{prd},x_{2}^{gt})$\\
$y_{1}^{\mathcal{I}}=max(y_{1}^{prd},y_{1}^{gt})$, $y_{2}^{\mathcal{I}}=min(y_{2}^{prd},y_{2}^{gt})$
\begin{eqnarray}
  \mathcal{I}=
  \begin{cases}
    (x_{2}^{\mathcal{I}}-x_{1}^{\mathcal{I}})*(y_{2}^{\mathcal{I}}-y_{1}^{\mathcal{I}}),& if x_{2}^{\mathcal{I}}>x_{1}^{\mathcal{I}}$, $y_{2}^{\mathcal{I}}>y_{1}^{\mathcal{I}}\\
    0,&otherwise.\\
  \end{cases}
  \nonumber
\end{eqnarray}

% \STATE Finding the coordinate of smallest enclosing box $B^{c}$:\\
% $x_{1}^{c}=min(\hat{x}_{1}^{prd},{x}_{1}^{gt})$, $x_{2}^{c}=max(\hat{x}_{2}^{prd},{x}_{2}^{gt})$\\
% $y_{1}^{c}=min(\hat{y}_{1}^{prd},{y}_{1}^{gt})$, $y_{2}^{c}=max(\hat{y}_{2}^{prd},{y}_{2}^{gt})$
% \STATE Calculating area of $B^{c}$: $A^{c}=(x_{2}^{c}-x_{1}^{c})*(y_{2}^{c}-y_{1}^{c})$
% \STATE Calculating diagonal distance of $B^{c}$: $D^{c}=\sqrt{(x_{2}^{c}-x_{1}^{c})^{2}*(y_{2}^{c}-y_{1}^{c})^{2}}$  
\STATE $IoU=\frac{\mathcal{I}}{\mathcal{U} }$, where $\mathcal{U}=A^{gt}+A^{prd}-\mathcal{I}$

% \STATE $V = \frac{4}{\pi ^{2}}(arctan\frac{w_{gt}}{h_{gt}}-arctan\frac{w_{prd}}{h_{prd}})^2, \alpha =\frac{V}{1-IoU+V}$
% % \STATE $$
% \STATE $GIoU=IoU-\frac{A^{c}-\mathcal{U}}{A^{c}}, DIoU=IoU-\frac{\rho ^{2}(\mathcal{B}_{prd},\mathcal{B}_{gt})}{(D^{c})^{2}}$
% % \STATE $$
% \STATE $CIoU=IoU-\frac{\rho ^{2}(\mathcal{B}_{prd},\mathcal{B}_{gt})}{(D^{c})^{2}}-\alpha V$

% \STATE $EIoU=IoU-\frac{\rho ^{2}(\mathcal{B}_{prd},\mathcal{B}_{gt})}{(D^{c})^{2}}-\frac{\rho ^{2}(w^{prd},w^{gt})}{(x_{2}^{c}-x_{1}^{c})^{2}}-\frac{\rho ^{2}(h^{prd},h^{gt})}{(y_{2}^{c}-y_{1}^{c})^{2}}$

\STATE $MPDIoU=IoU-\frac{d_{1}^{2}}{h^2+w^2}-\frac{d_{2}^{2}}{h^2+w^2}$
\STATE $\mathcal{L}_{IoU}=1-IoU$, $\mathcal{L}_{MPDIoU}=1-MPDIoU$.
% $\mathcal{L}_{GIoU}=1-GIoU$, $\mathcal{L}_{DIoU}=1-DIoU$, $\mathcal{L}_{CIoU}=1-CIoU$, $\mathcal{L}_{EIoU}=1-EIoU$, 
\end{algorithmic} 
\caption{IoU and MPDIoU as bounding box losses} %算法的题目 
\label{alg2} %算法的标签 
\end{algorithm}

% Since back-propagating min, max and piece-wise linear functions, e.g. Relu, are feasible, it can be shown that every component in Alg. \ref{alg2} has a well-behaved derivative. Therefore, $IoU$, $GIoU$, $DIoU$, $CIoU$, $EIoU$ and $MPDIoU$ can be directly used as loss functions, i.e. $\mathcal{L}_{GIoU}$, $\mathcal{L}_{DIoU}$, $\mathcal{L}_{CIoU}$, $\mathcal{L}_{EIoU}$ and $\mathcal{L}_{MPDIoU}$, for optimizing object detection or instance segmentation models. In this case, we are directly optimizing a metric as loss, which is an optimal choice for the metric. However, in all non-overlapping cases, $IoU$ has zero gradient, which affects both training accuracy and convergence speed. $MPDIoU$, in contrast, has a gradient in all possible cases, including non-overlapping situations. 

Considering the groundtruth bounding box, $\mathcal{B}_{gt}$ is a rectangle with area bigger than zero, i.e. $A^{gt} > 0$. Alg. \ref{alg2} (1) and the Conditions in Alg. \ref{alg2} (6) respectively ensure the predicted area $A^{prd}$ and intersection area $\mathcal{I}$ are non-negative values, i.e. $A^{prd} \geq 0$ and $\mathcal{I}\geq 0 $, ${\forall}\mathcal{B}_{prd} \in \mathbb{R}^{4}$. Therefore union area $\mathcal{U}>0$ for any predicted bounding box $\mathcal{B}_{prd}=(x_{1}^{prd},y_{1}^{prd},x_{2}^{prd},y_{2}^{prd})\in \mathbb{R}^{4}$. This ensures that the denominator in $IoU$ cannot be zero for any predicted value of outputs. In addition, for any values of $\mathcal{B}_{prd}=(x_{1}^{prd},y_{1}^{prd},x_{2}^{prd},y_{2}^{prd})\in \mathbb{R}^{4}$, the union area is always bigger than the intersection area, i.e. $ \mathcal{U} \geq\mathcal{I}$. As a result, $\mathcal{L}_{MPDIoU}$ is always bounded, i.e. $0\leq\mathcal{L}_{MPDIoU}< 3, \forall\mathcal{B}_{prd} \in \mathbb{R}^{4}$.

$\mathcal{L}_{MPDIoU}$ \textbf{behaviour when $IoU$ = 0:} For $MPDIoU$ loss, we have $\mathcal{L}_{MPDIoU} =1-MPDIoU=1+\frac{d_{1}^{2}}{d^2}+\frac{d_{2}^{2}}{d^2}-IoU$. In the case of $\mathcal{B}_{gt}$ and $\mathcal{B}_{prd}$ do not overlap, which means $IoU=0$, $MPDIoU$ loss can be simplified to $\mathcal{L}_{MPDIoU} =1-MPDIoU=1+\frac{d_{1}^{2}}{d^2}+\frac{d_{2}^{2}}{d^2}$. In this case, by minimizing $\mathcal{L}_{MPDIoU}$, we actually minimize $\frac{d_{1}^{2}}{d^2}+\frac{d_{2}^{2}}{d^2}$. This term is a normalized measure between 0 and 1, $i.e. 0\leq \frac{d_{1}^{2}}{d^2}+\frac{d_{2}^{2}}{d^2}< 2$.

\section{Experimental Results}
We evaluate our new bounding box regression loss $\mathcal{L}_{MPDIoU}$ by incorporating it into the most popular 2D object detector and instance segmentation models such as YOLO v7 \cite{2022YOLOv7} and YOLACT \cite{Bolya_2019_ICCV}. To this end, we replace their default regression losses with $\mathcal{L}_{MPDIoU}$ , i.e. we replace $\ell_{1}$-smooth in YOLACT \cite{Bolya_2019_ICCV} and $\mathcal{L}_{CIoU}$ in YOLO v7 \cite{2022YOLOv7}. We also compare the baseline losses against $\mathcal{L}_{GIoU}$.
%-------------------------------------------------------------------------
\subsection{Experimental Settings}
The experimental environment can be summarized as follows: the memory is 32GB, the operating system is windows 11, the CPU is Intel i9-12900k, and the graphics card is NVIDIA Geforce RTX 3090 with 24GB memory. In order to conduct a fair comparison, all of the experiments are implemented with PyTorch \cite{2017Automatic}. 
\subsection{Datasets}
We train all object detection and instance segmentation baselines and report all the results on two standard benchmarks, i.e. the PASCAL VOC \cite{2015The} and the Microsoft Common Objects in Context (MS COCO 2017) \cite{2014Microsoft} challenges. The details of their training protocol and their evaluation will be explained in their own sections.\\
\textbf{PASCAL VOC 2007\&2012}: The Pascal Visual Object Classes (VOC) \cite{2015The} benchmark is one of the most widely used datasets for classification, object detection and semantic segmentation, which contains about 9963 images. The training dataset and the test dataset are 50\% for each, where objects from 20 pre-defined categories are annotated with horizontal bounding boxes. Due to the small scale of images for instance segmentation, which leads to weak performance, we only provide the instance segmentation results training with MS COCO 2017.\\
\textbf{MS COCO}: MS COCO \cite{2014Microsoft} is a widely used benchmark for image captioning, object detection and instance segmentation, which contains more than 200,000 images across train, validation and test sets with over 500,000 annotated object instances from 80 categories.\\
\textbf{IIIT5k}: IIIT5k \cite{IIIT5K-Words} is one of the popular scene text spotting benchmark with character-level annotations, which contains 5,000 cropped word images collected from the Internet. The character category includes English letters and digits. There are 2,000 images for training and 3,000 images for testing.\\
\textbf{MTHv2}: MTHv2 \cite{MTHv2} is one of the popular OCR benchmark with character-level annotations. The character category includes simplified and traditional characters. It contains more than 3000 images of Chinese historical documents and more than 1 million Chinese characters.\\
\subsection{Evaluation Protocol}
In this paper, we used the same performance measure as the MS COCO 2018 Challenge \cite{2014Microsoft} to report all of our results, including mean Average Precision (mAP) over different class labels for a specific value of $IoU$ threshold in order to determine true positives and false positives. The main performance measure of object detection used in our experiments is shown by precision and mAP@0.5:0.95. We report the mAP value for $IoU$ thresholds equal to 0.75, shown as AP75 in the tables.
As for instance segmentation, the main performance measure used in our experiments are shown by AP and AR, which is averaging mAP and mAR across different value of $IoU$ thresholds,  $i.e. IoU = \{ .5, .55,..., .95\} $. 

All of the object detection and instance segmentation baselines have also been evaluated using the test set of the MS COCO 2017 and PASCAL VOC 2007\&2012. The results will be shown in following section.
\subsection{Experimental Results of Object Detection}

\textbf{Training protocol.} We used the original Darknet implementation of YOLO v7 released by \cite{2022YOLOv7}. As for baseline results (training using $GIoU$ loss), we selected DarkNet-608 as backbone in all experiments and followed exactly their training protocol using the reported default parameters and the number of iteration on each benchmark. To train YOLO v7 using $GIoU$, $DIoU$, $CIoU$, $EIoU$ and $MPDIoU$ losses, we simply replace the bounding box regression $IoU$ loss with $\mathcal{L}_{GIoU}$, $\mathcal{L}_{DIoU}$, $\mathcal{L}_{CIoU}$, $\mathcal{L}_{EIoU}$ and $\mathcal{L}_{MPDIoU}$ losses explained in \ref{alg2}.
\begin{figure*}[h]
  \captionsetup[subfloat]{labelformat=empty} 
  \subfloat[]{
    
    \includegraphics[width=0.19\linewidth]{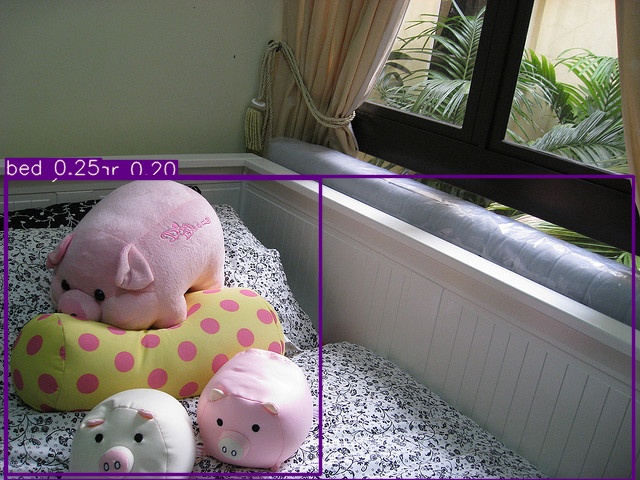}

  }
  \subfloat[]{
    \includegraphics[width=0.19\linewidth]{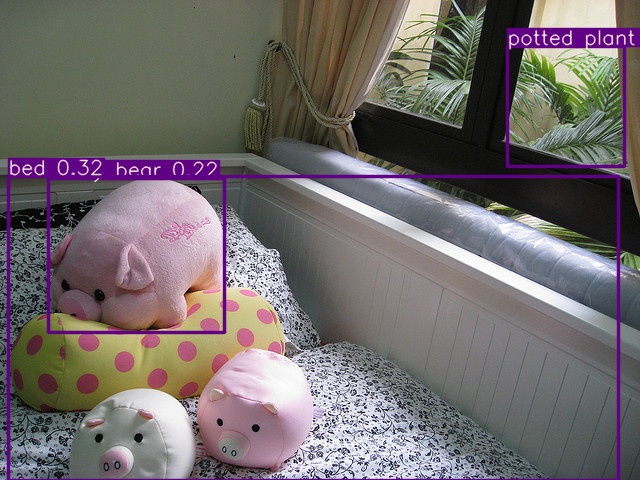}
  
  }
  \subfloat[]{
    \includegraphics[width=0.19\linewidth]{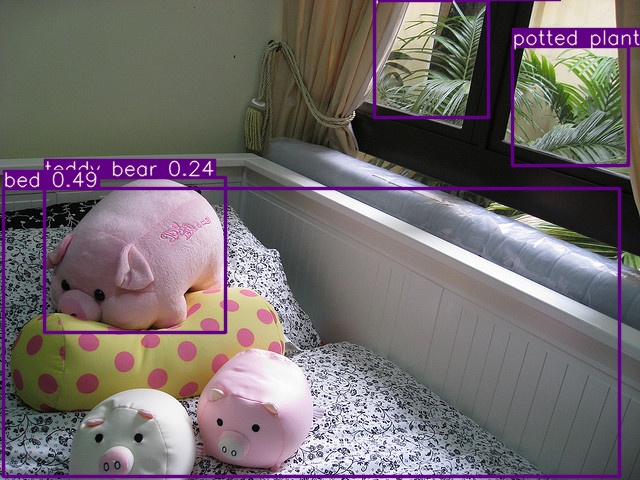}
  }
  \subfloat[]{
    \includegraphics[width=0.19\linewidth]{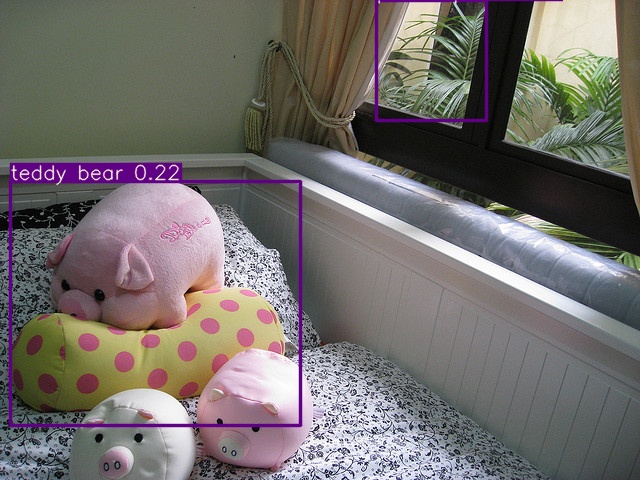}

  }
  \subfloat[]{
    \includegraphics[width=0.19\linewidth]{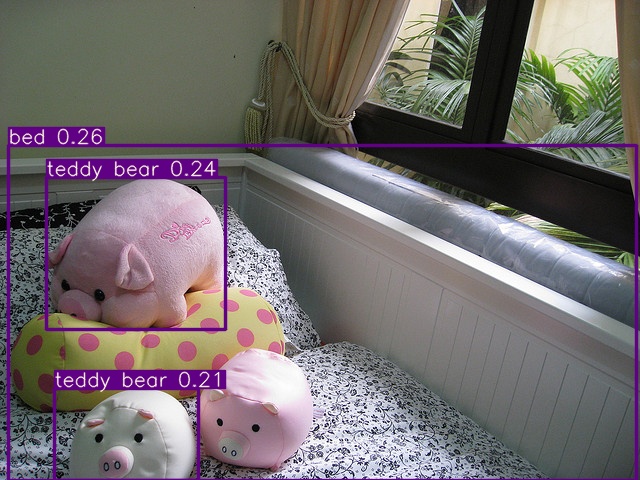}

  }
  \vspace{-10pt}
  \subfloat[$\mathcal{L}_{GIoU}$]{
    \includegraphics[width=0.19\linewidth]{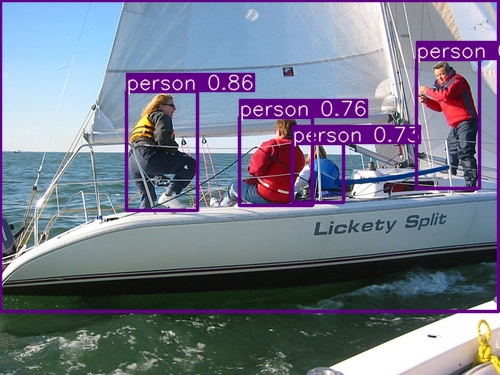}

  }
  \subfloat[$\mathcal{L}_{DIoU}$]{
    \includegraphics[width=0.19\linewidth]{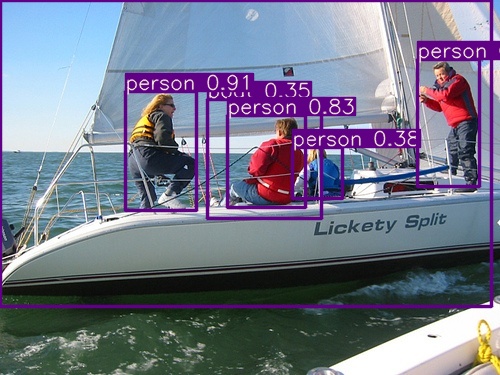}
  
  }
  \subfloat[$\mathcal{L}_{CIoU}$]{
    \includegraphics[width=0.19\linewidth]{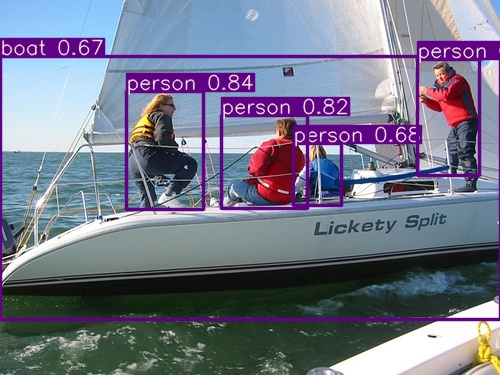}
  }
  \subfloat[$\mathcal{L}_{EIoU}$]{
    \includegraphics[width=0.19\linewidth]{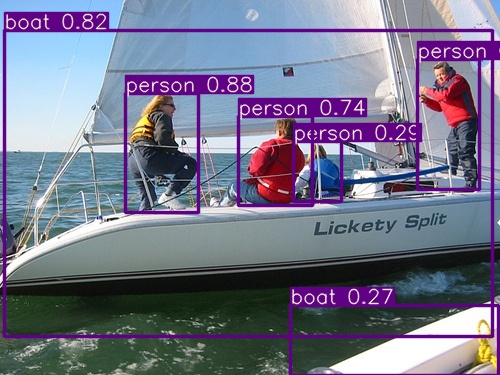}

  }
  \subfloat[$\mathcal{L}_{MPDIoU}$]{
    \includegraphics[width=0.19\linewidth]{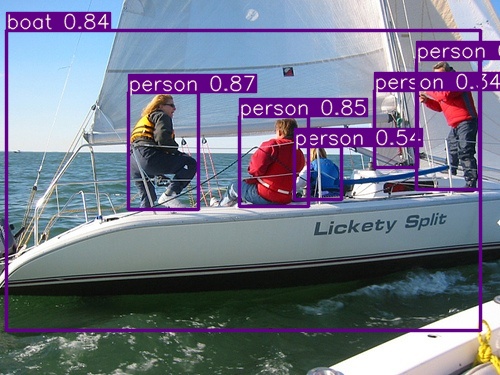}
  }
  \caption{Object detection results from \textbf{the test set of MS COCO 2017} \cite{2014Microsoft} and \textbf{PASCAL VOC 2007} \cite{2015The} using YOLO v7 \cite{2022YOLOv7} trained using (left to right) $\mathcal{L}_{GIoU}$, $\mathcal{L}_{DIoU}$, $\mathcal{L}_{CIoU}$, $\mathcal{L}_{EIoU}$ and $\mathcal{L}_{MPDIoU}$ losses.}
  \centering
  \label{resultVOC}
\end{figure*}

\begin{figure*}[h]
  \begin{minipage}{0.45\linewidth}
    % \begin{table}[h]
    %   \centering
    \begin{tabular}{c|c|c}
    \toprule
    \diagbox{Loss}{Evaluation}&
    AP&AP75\\
    % &IoU&MPDIoU&IoU&MPDIoU\\
    \midrule
      $\mathcal{L}_{GIoU}$&56.1
      &61.4
      \\
      \hline
      $\mathcal{L}_{DIoU}$&56.2
      &61.5\\
      Relative improv(\%)&0.17&0.16\\
      \hline
      $\mathcal{L}_{CIoU}$&56.2
      &61.6\\
      Relative improv(\%)&0.17&0.32\\
      \hline
      $\mathcal{L}_{EIoU}$&56.3
      &61.6\\
      Relative improv(\%)&0.35&0.32\\
      \hline
      $\mathcal{L}_{MPDIoU}$&\textbf{57.3}
      &\textbf{62}\\
      Relative improv(\%)&\textbf{2.13}&\textbf{0.97}\\
      \bottomrule
    \end{tabular}
    \captionof{table}{Comparison between the performance of YOLO v7 \cite{2022YOLOv7} trained using its own loss ($\mathcal{L}_{CIoU}$) as well as $\mathcal{L}_{GIoU}$, $\mathcal{L}_{DIoU}$, $\mathcal{L}_{EIoU}$ and $\mathcal{L}_{MPDIoU}$ losses. The results are reported on the \textbf{test set of PASCAL VOC 2007\&2012}.}
    \label{yolov7VOC}
    % \end{table}
  \end{minipage}
  \begin{minipage}{0.55\linewidth}
    \centering
      \includegraphics[width=0.49\linewidth]{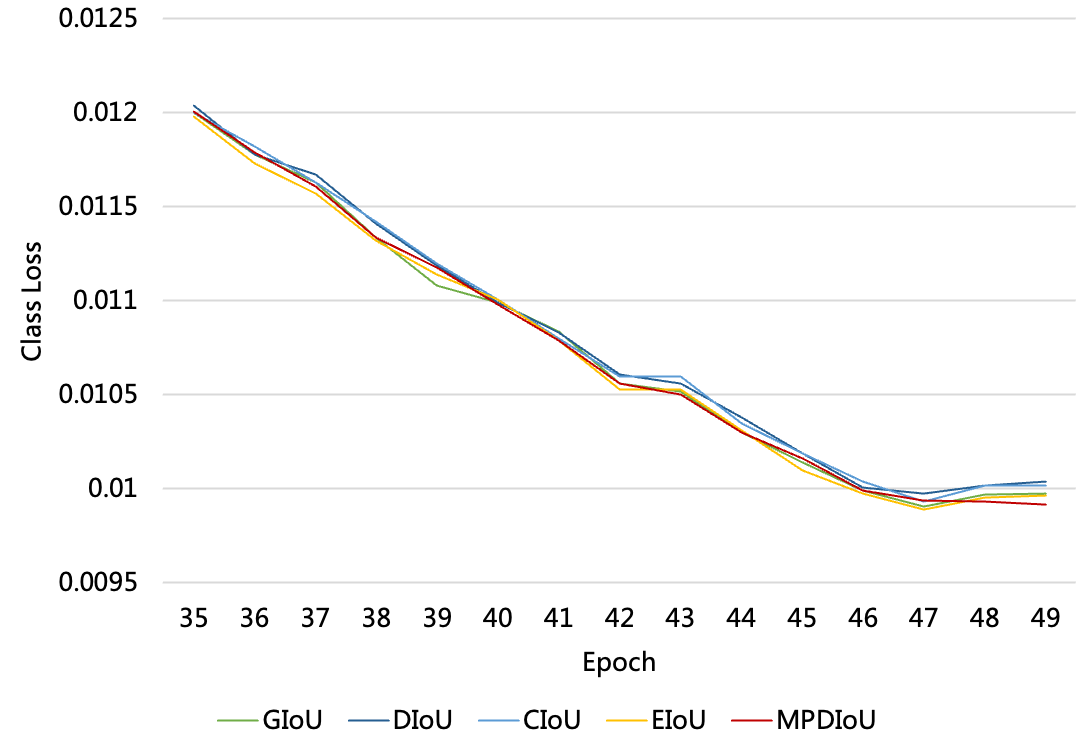}
      \includegraphics[width=0.49\linewidth]{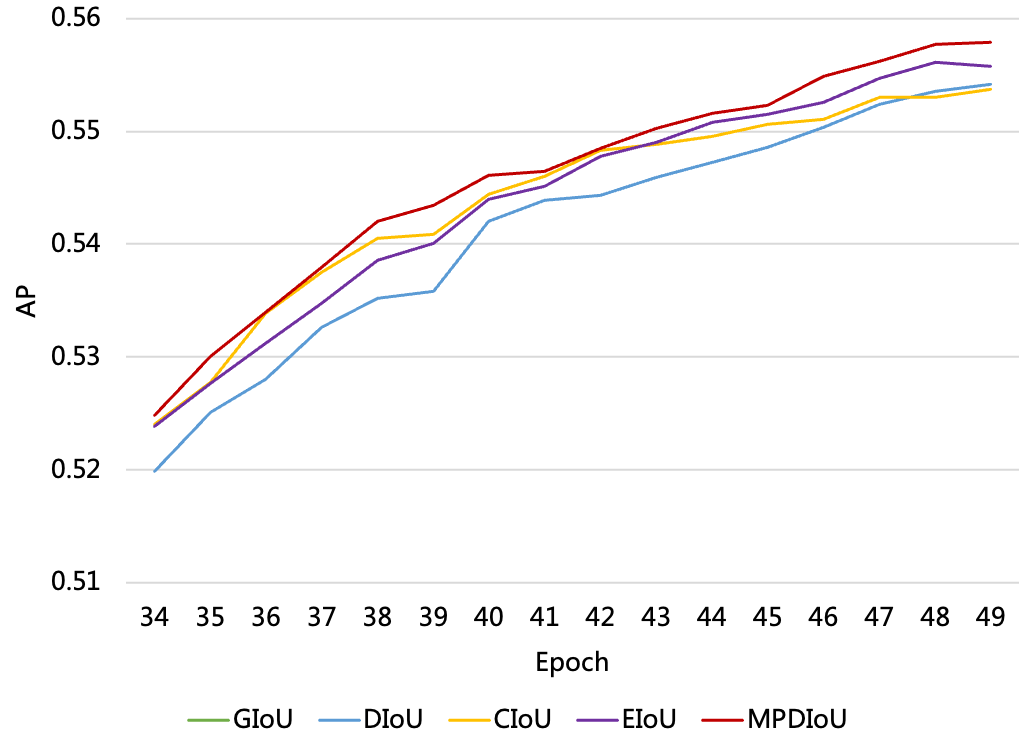}
      \caption{The bbox loss and AP values against training iterations when YOLO v7 \cite{2022YOLOv7} was trained on PASCAL VOC 2007\&2012 \cite{2015The} using $\mathcal{L}_{GIoU}$, $\mathcal{L}_{DIoU}$, $\mathcal{L}_{CIoU}$, $\mathcal{L}_{EIoU}$ and $\mathcal{L}_{MPDIoU}$ losses.}
      \label{VOC_Train}
    
  \end{minipage}
\end{figure*}

  Following the original code's training protocol, we trained YOLOv7 \cite{2022YOLOv7} using each loss on both training and validation set of the dataset up to 150 epochs. We set the patience of early stop mechanism as 5 to reduce the training time and save the model with the best performance. Their performance using the best checkpoints for each loss has been evaluated on the test set of PASCAL VOC 2007\&2012. The results have been reported in Table \ref{yolov7VOC}.

% \begin{table}[ht]
%   \caption{Comparison between the performance of YOLO v7 \cite{2022YOLOv7} trained using its own loss ($\mathcal{L}_{CIoU}$) as well as $\mathcal{L}_{GIoU}$, $\mathcal{L}_{DIoU}$, $\mathcal{L}_{EIoU}$ and $\mathcal{L}_{MPDIoU}$ losses. The results are reported on the \textbf{test set of MS COCO 2017}.}
%   \label{yolov7COCO}
%   \centering
%   \begin{tabular}{c|cc|cc}
%   \toprule
%   \diagbox{Loss}{Evaluation}&AP&AP75&AR$_{100}$&AR$_{L}$\\
%   \midrule
%     $\mathcal{L}_{GIoU}$&38.9&41.8&55&73.2
%     \\
%     \hline
%     $\mathcal{L}_{DIoU}$&39.1&41.8&\textbf{55.1}&73.5\\
%     Relative improv(\%)&0.51&0&\textbf{0.18}&0.4\\
%     \hline
%     $\mathcal{L}_{CIoU}$&39&41.8&54.8&73.5\\
%     Relative improv(\%)&0.25&0&-0.51&0.4\\
%     \hline
%     $\mathcal{L}_{EIoU}$&39.1
%     &41.8&54.9&73.2\\
%     Relative improv(\%)&1.78&0&0.4&0\\
%     \hline
%     $\mathcal{L}_{MPDIoU}$&\textbf{40.1}&\textbf{41.8}&55&\textbf{73.6}\\
%     Relative improv(\%)&0.25&0&0&0.54\\
%     \bottomrule
%   \end{tabular}
  
% \end{table}

\subsection{Experimental Results of Character-level Scene Text Spotting}
\textbf{Training protocol.} We used the similar training protocol with the experiments of object detection. Following the original code's training protocol, we trained YOLOv7 \cite{2022YOLOv7} using each loss on both training and validation set of the dataset up to 30 epochs. Their performance using the best checkpoints for each loss has been evaluated using the test set of IIIT5K \cite{IIIT5K-Words} and MTHv2 \cite{yang2018SCUT}. The results have been reported in Table \ref{yolov7IIIT5k} and Table \ref{yolov7MTHv2}.\\

\begin{figure*}[h]
  \captionsetup[subfloat]{labelformat=empty} 
  
  \subfloat[]{
    \includegraphics[width=0.19\linewidth]{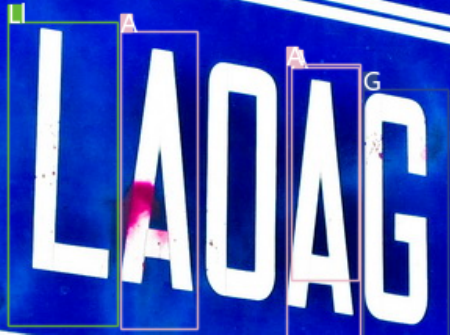}
  }
  \subfloat[]{
    \includegraphics[width=0.19\linewidth]{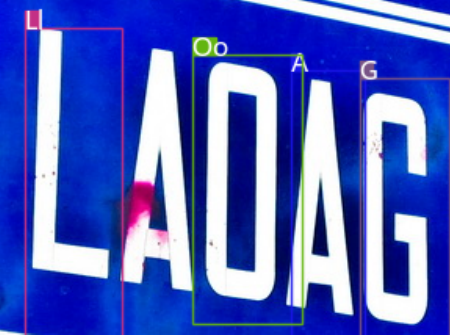}
  
  }
  \subfloat[]{
    \includegraphics[width=0.19\linewidth]{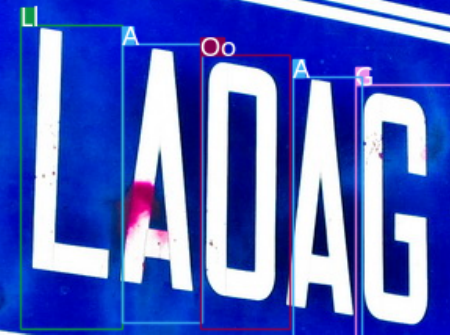}
  }
  \subfloat[]{
    \includegraphics[width=0.19\linewidth]{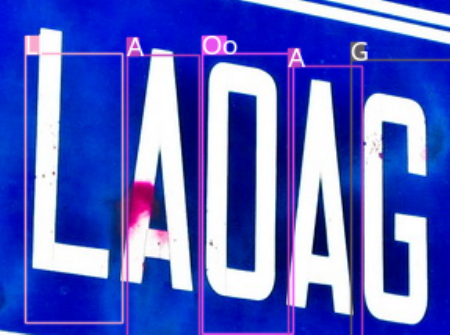}
  }
  \subfloat[]{
    \includegraphics[width=0.19\linewidth]{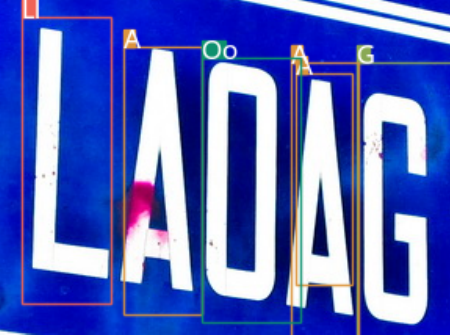}
  }
  \vspace{-10pt}
  
  \subfloat[$\mathcal{L}_{GIoU}$]{
    \includegraphics[width=0.19\linewidth]{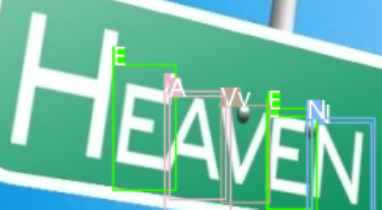}

  }
  \subfloat[$\mathcal{L}_{DIoU}$]{
    \includegraphics[width=0.19\linewidth]{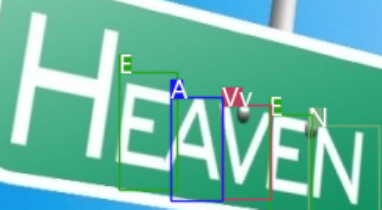}
  
  }
  \subfloat[$\mathcal{L}_{CIoU}$]{
    \includegraphics[width=0.19\linewidth]{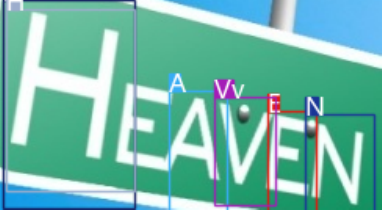}
  }
  \subfloat[$\mathcal{L}_{EIoU}$]{
    \includegraphics[width=0.19\linewidth]{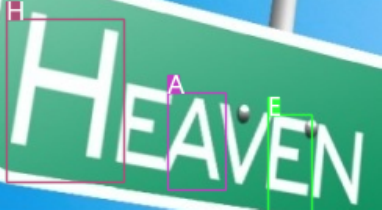}
  }
  \subfloat[$\mathcal{L}_{MPDIoU}$]{
    \includegraphics[width=0.19\linewidth]{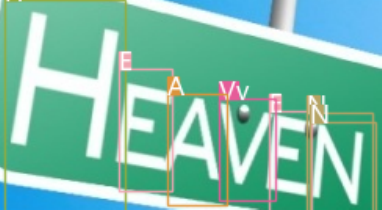}
  }
  \caption{Character-level scene text spotting results from \textbf{the test set of IIIT5K} \cite{IIIT5K-Words} using YOLOv7 \cite{2022YOLOv7}trained using (left to right) $\mathcal{L}_{GIoU}$, $\mathcal{L}_{DIoU}$, $\mathcal{L}_{CIoU}$, $\mathcal{L}_{EIoU}$ and $\mathcal{L}_{MPDIoU}$ losses.}
  \centering
  \label{resultIIIT5k}
\end{figure*}

\begin{minipage}{0.45\linewidth}
  \centering
\begin{tabular}{c|c|c}
\toprule
\diagbox{Loss}{Evaluation}&
AP&AP75\\
% &IoU&MPDIoU&IoU&MPDIoU\\
\midrule
  $\mathcal{L}_{GIoU}$&42.9&45
  
  \\
  \hline
  $\mathcal{L}_{DIoU}$&42.2&42.3
  \\
  Relative improv(\%)&-1.6&-6\\
  \hline
  $\mathcal{L}_{CIoU}$&44.1
  &46.6\\
  Relative improv(\%)&2.7&3.5\\
  \hline
  $\mathcal{L}_{EIoU}$&41
  &42.6\\
  Relative improv(\%)&-4.4&-5.3\\
  \hline
  $\mathcal{L}_{MPDIoU}$&\textbf{44.5}
  &\textbf{46.6}\\
  Relative improv(\%)&\textbf{3.7}&\textbf{3.5}\\
  \bottomrule
\end{tabular}
\captionof{table}{Comparison between the performance of YOLO v7 \cite{2022YOLOv7} trained using its own loss ($\mathcal{L}_{CIoU}$) as well as $\mathcal{L}_{GIoU}$, $\mathcal{L}_{DIoU}$, $\mathcal{L}_{EIoU}$ and $\mathcal{L}_{MPDIoU}$ losses. The results are reported on the \textbf{test set of IIIT5K}.}
\label{yolov7IIIT5k}
\end{minipage}
\begin{minipage}{0.45\linewidth}
  \centering
\begin{tabular}{c|c|c}
\toprule
\diagbox{Loss}{Evaluation}&
AP&AP75\\
% &IoU&MPDIoU&IoU&MPDIoU\\
\midrule
  $\mathcal{L}_{GIoU}$&52.1&55.3\\
  \hline
  $\mathcal{L}_{DIoU}$&53.2&55.8\\
  Relative improv(\%)&2.1&0.9\\
  \hline
  $\mathcal{L}_{CIoU}$&52.3
  &53.6\\
  Relative improv(\%)&0.3&-3.0\\
  \hline
  $\mathcal{L}_{EIoU}$&53.2
  &54.7\\
  Relative improv(\%)&2.1&-1.0\\
  \hline
  $\mathcal{L}_{MPDIoU}$&\textbf{54.5}
  &\textbf{58}\\
  Relative improv(\%)&\textbf{4.6}&\textbf{4.8}\\
  \bottomrule
\end{tabular}
\captionof{table}{Comparison between the performance of YOLO v7 \cite{2022YOLOv7} trained using its own loss ($\mathcal{L}_{CIoU}$) as well as $\mathcal{L}_{GIoU}$, $\mathcal{L}_{DIoU}$, $\mathcal{L}_{EIoU}$ and $\mathcal{L}_{MPDIoU}$ losses. The results are reported on the \textbf{test set of MTHv2}.}
\label{yolov7MTHv2}
\end{minipage}

As we can see, the results in Tab. \ref{yolov7IIIT5k} and \ref{yolov7MTHv2} show that training YOLO v7 using $\mathcal{L}_{MPDIoU}$ as regression loss can considerably improve its performance compared to the existing regression losses including $\mathcal{L}_{GIoU}$, $\mathcal{L}_{DIoU}$, $\mathcal{L}_{CIoU}$, $\mathcal{L}_{EIoU}$. Our proposed $\mathcal{L}_{MPDIoU}$ shows outstanding performance on character-level scene text spotting.

\subsection{Experimental Results of Instance Segmentation}
\begin{figure*}[h]
   
\end{figure*}
\textbf{Training protocol}.
We used the latest PyTorch implementations of YOLACT \cite{Bolya_2019_ICCV},
released by University of California. For baseline results
(trained using $\mathcal{L}_{GIoU}$), we selected ResNet-50 as the backbone network architecture for both YOLACT in all experiments and followed their training protocol using the reported default parameters and the number of iteration on each benchmark. To train YOLACT using $GIoU$, $DIoU$, $CIoU$, $EIoU$ and $MPDIoU$ losses, we replaced
their $\ell_{1}$-smooth loss in the final bounding box refinement
stage with $\mathcal{L}_{GIoU}$, $\mathcal{L}_{DIoU}$, $\mathcal{L}_{CIoU}$, $\mathcal{L}_{EIoU}$ and $\mathcal{L}_{MPDIoU}$  losses explained in \ref{alg2}. Similar with the YOLO v7 experiment, we replaced the original losses for bounding box regression with our proposed $\mathcal{L}_{MPDIoU}$.

\begin{figure}[ht]
  \subfloat[]{\includegraphics[width=0.33\linewidth]{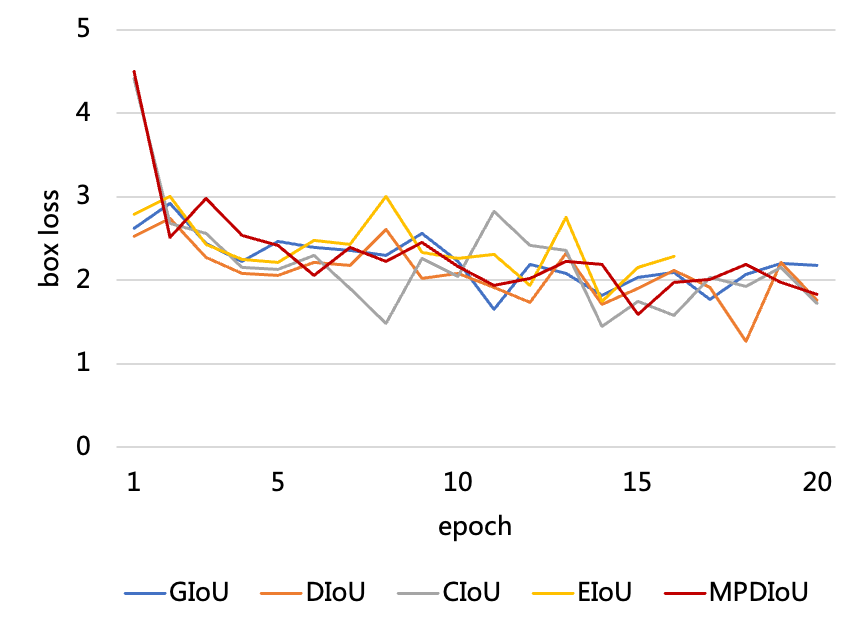}}
  \subfloat[]{\includegraphics[width=0.33\linewidth]{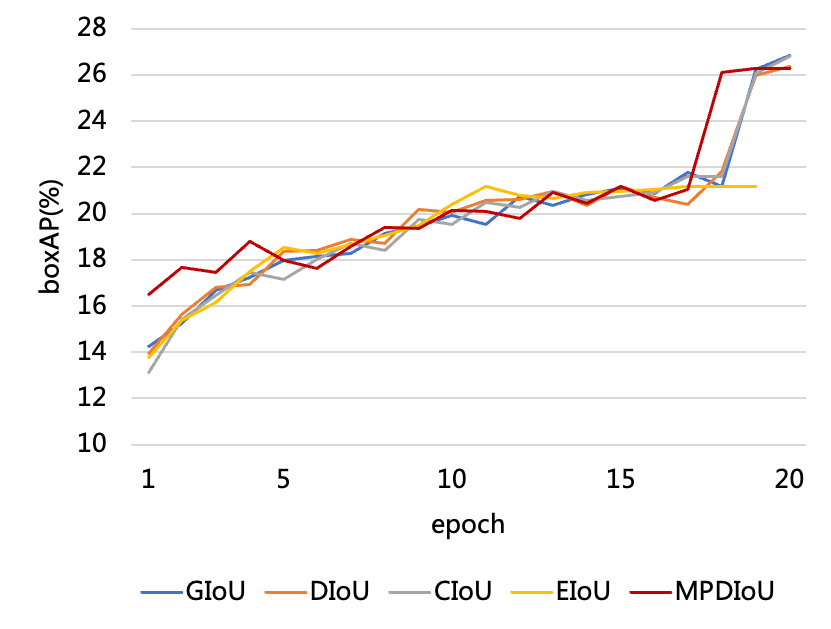}}
  \subfloat[]{\includegraphics[width=0.33\linewidth]{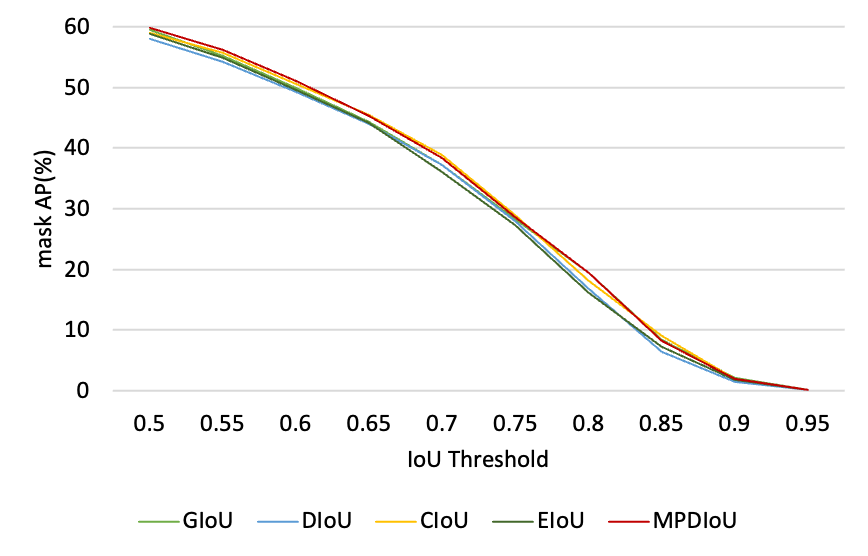}}
  \caption{The bbox loss and boxAP values against training iterations when YOLACT \cite{Bolya_2019_ICCV} was trained on MS COCO 2017 \cite{2014Microsoft} using $\mathcal{L}_{GIoU}$, $\mathcal{L}_{DIoU}$, $\mathcal{L}_{CIoU}$, $\mathcal{L}_{EIoU}$ and $\mathcal{L}_{MPDIoU}$ losses and the mask AP value against different IoU thresholds.}
  \label{COCO_Train}
\end{figure}

\begin{figure*}[h]
  \captionsetup[subfloat]{labelformat=empty} 
  
  \subfloat[]{
    \includegraphics[width=0.19\linewidth]{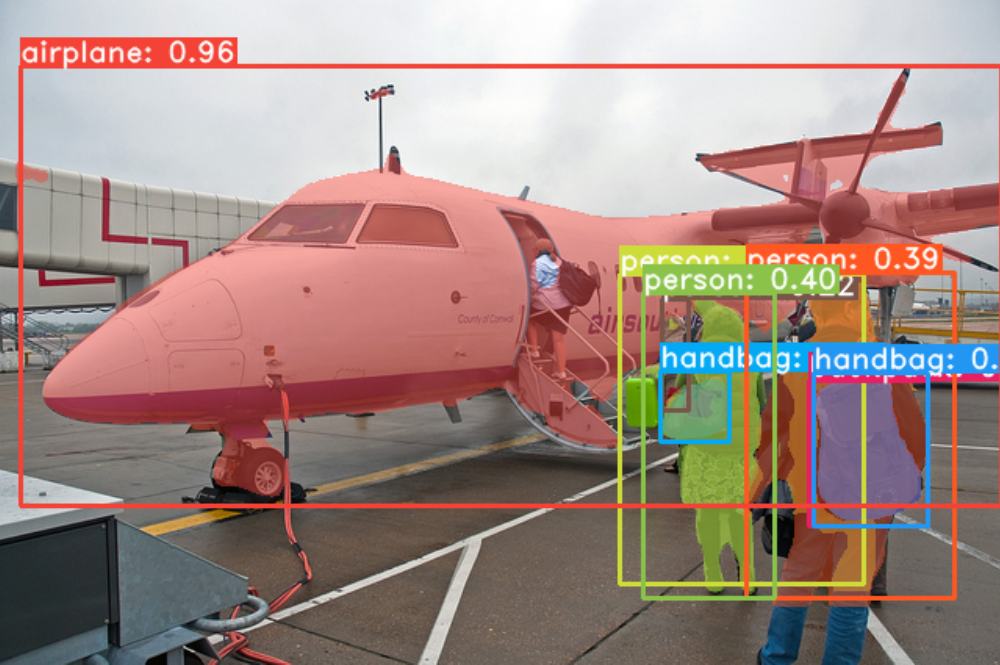}
  }
  \subfloat[]{
    \includegraphics[width=0.19\linewidth]{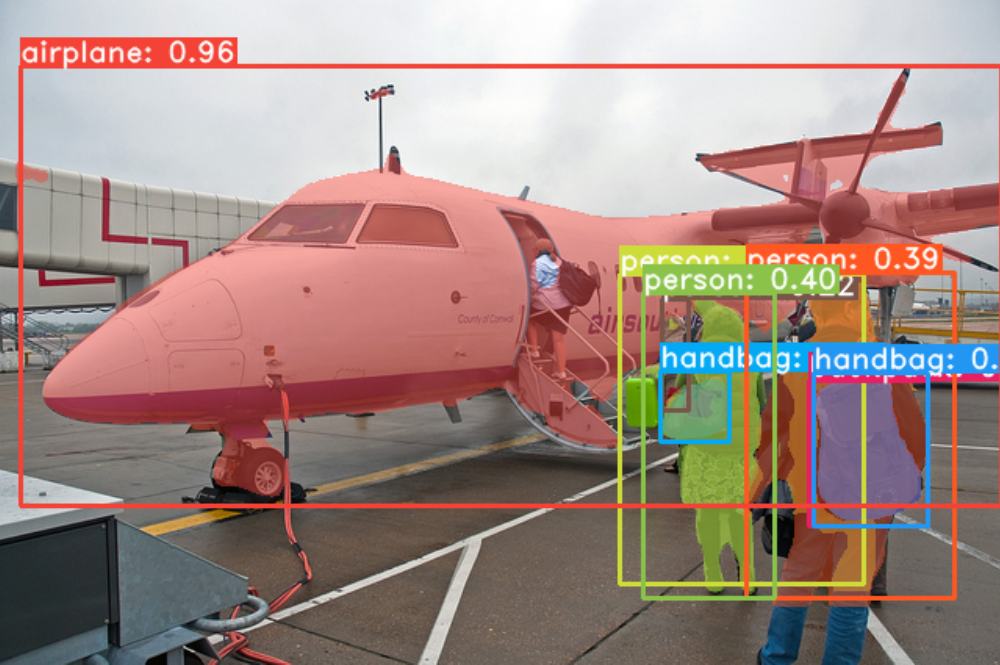}
  
  }
  \subfloat[]{
    \includegraphics[width=0.19\linewidth]{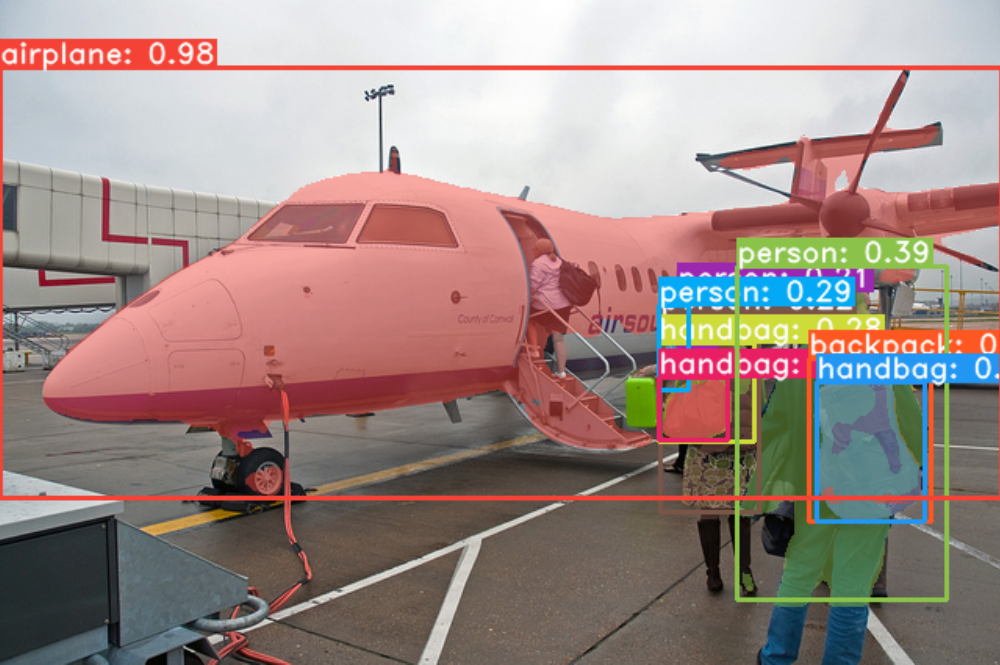}
  }
  \subfloat[]{
    \includegraphics[width=0.19\linewidth]{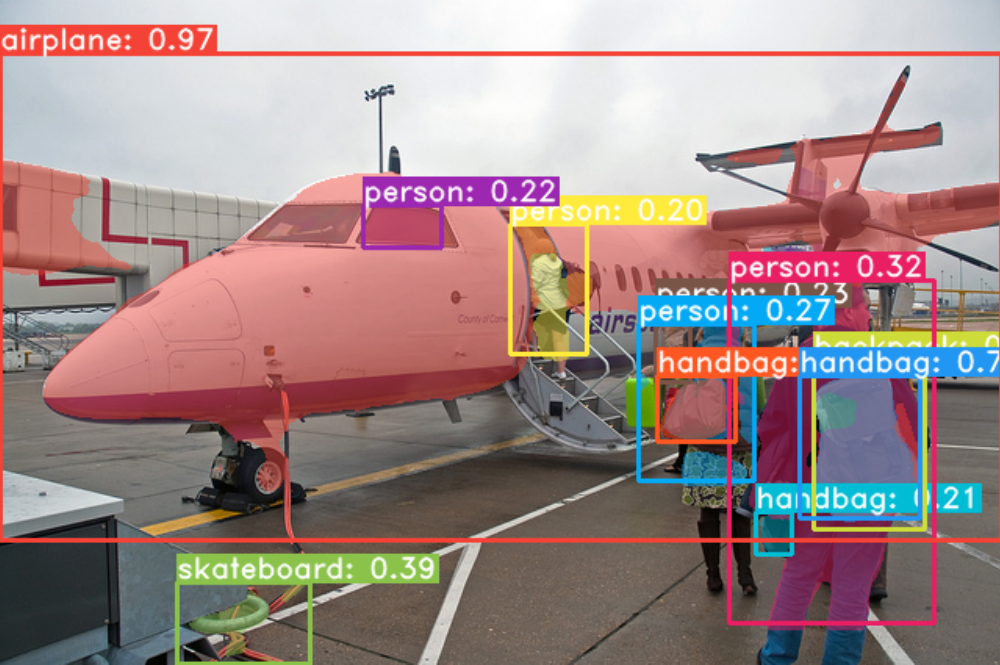}
  }
  \subfloat[]{
    \includegraphics[width=0.19\linewidth]{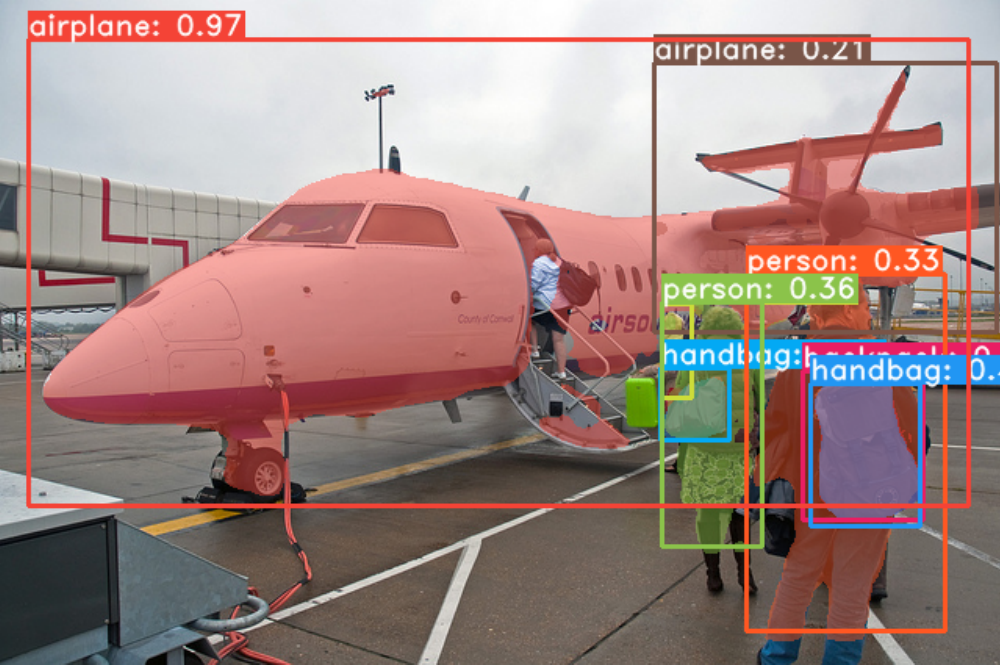}
  }
  \vspace{-10pt}
  
  \subfloat[$\mathcal{L}_{GIoU}$]{
    \includegraphics[width=0.19\linewidth]{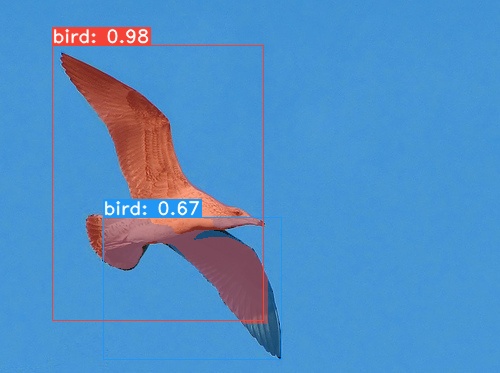}

  }
  \subfloat[$\mathcal{L}_{DIoU}$]{
    \includegraphics[width=0.19\linewidth]{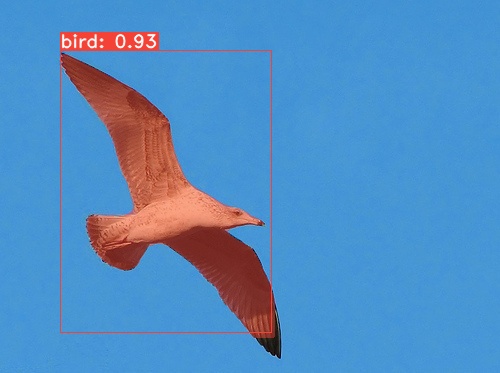}
  
  }
  \subfloat[$\mathcal{L}_{CIoU}$]{
    \includegraphics[width=0.19\linewidth]{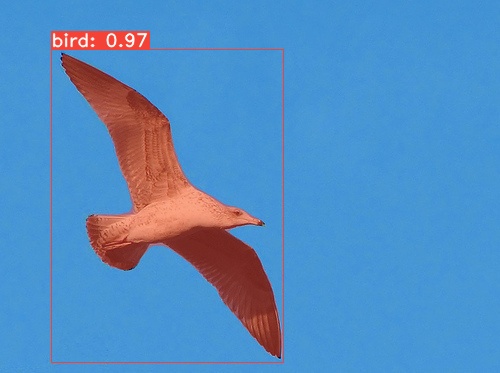}
  }
  \subfloat[$\mathcal{L}_{EIoU}$]{
    \includegraphics[width=0.19\linewidth]{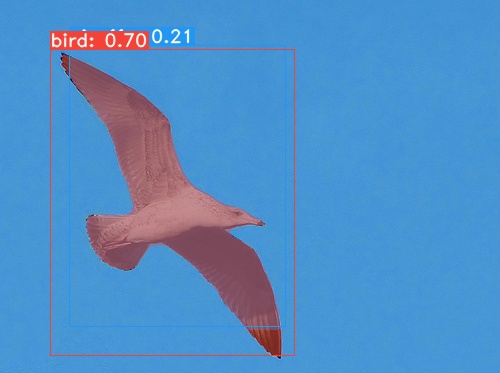}
  }
  \subfloat[$\mathcal{L}_{MPDIoU}$]{
    \includegraphics[width=0.19\linewidth]{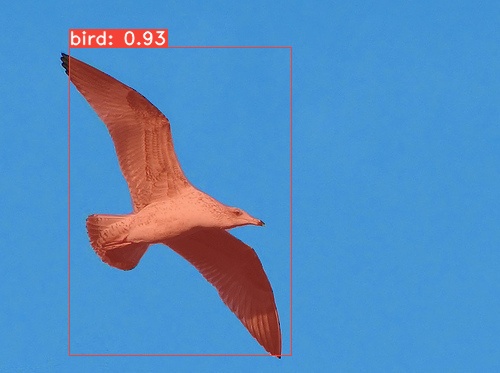}
  }
  \caption{Instance segmentation results from \textbf{the test set of MS COCO 2017} \cite{2014Microsoft} and \textbf{PASCAL VOC 2007} \cite{2015The} using YOLACT \cite{Bolya_2019_ICCV} trained using (left to right) $\mathcal{L}_{GIoU}$, $\mathcal{L}_{DIoU}$, $\mathcal{L}_{CIoU}$, $\mathcal{L}_{EIoU}$ and $\mathcal{L}_{MPDIoU}$ losses.}
  \centering
  \label{resultCOCO}
\end{figure*}
\begin{table*}[ht]
  \caption{Instance segmentation results of YOLACT \cite{Bolya_2019_ICCV}. The models are retrained using $\mathcal{L}_{GIoU}$, $\mathcal{L}_{DIoU}$, $\mathcal{L}_{CIoU}$ and $\mathcal{L}_{EIoU}$ by us, and the results are reported on test set of MS COCO 2017 \cite{2014Microsoft}. FPS and time were recorded during training period.}
  \label{resultCOCOIns}
  \centering
  \begin{tabular}{c|cc|cccccc|cccccc}
  \toprule
  Loss&FPS&Time&AP&$AP_{50}$&$AP_{75}$&$AP_{S}$&$AP_{M}$&$AP_{L}$&$AR_{1}$&$AR_{10}$&$AR_{100}$&$AR_{S}$&$AR_{M}$&$AR_{L}$\\

  \midrule
    $\mathcal{L}_{GIoU}$
    &25.69&38.91&25&42.3&25.7&7.4&25.8&\textbf{39.7}&\textbf{24.7}&35.2&36.2&15.5&38.5&\textbf{53.3}\\
    
    $\mathcal{L}_{DIoU}$&25.90&38.61&25&42.2&25.6&7.5&25.7&39.6&24.5&35&35.9&14.9&38.6&52.9\\
    
    $\mathcal{L}_{CIoU}$&26.94&37.12&24.8&42.1&25.4&7.6&25.5&39&24.5&35.1&36.1&15.7&38.6&52.5\\
   
    $\mathcal{L}_{EIoU}$&25.71&38.90&20.1&35.6&19.9&5.4&19.6&32.4&20.8&29.8&30.6&12.3&31.8&45.1
    \\
    
    $\mathcal{L}_{MPDIoU}$&\textbf{27.11}
    &\textbf{36.89}&\textbf{25.1}&\textbf{42.4}&\textbf{25.8}&\textbf{7.6}&\textbf{25.8}&39.6&24.6&\textbf{35.3}&\textbf{36.3}&\textbf{15.8}&\textbf{38.9}&52.6
    \\
    
    \bottomrule
  \end{tabular}

\end{table*}
As Figure \ref{COCO_Train}(c) shows, incorporating $\mathcal{L}_{GIoU}$, $\mathcal{L}_{DIoU}$, $\mathcal{L}_{CIoU}$ and $\mathcal{L}_{EIoU}$ as the regression loss can slightly improve the performance of YOLACT on MS COCO 2017.  However, the improvement is obvious compared to the case where it is trained using $\mathcal{L}_{MPDIoU}$, where we visualized different values of mask AP against different value of $IoU$ thresholds, i.e. $0.5\leq IoU\leq 0.95$.

Similar to the above experiments, detection accuracy improves by using $\mathcal{L}_{MPDIoU}$ as regression loss over the existing loss functions. As Table \ref{resultCOCOIns} shows, our proposed $\mathcal{L}_{MPDIoU}$ performs better than existing loss functions on most of the metrics. However, the amount of improvement between different losses is less than previous experiments. This may be due to several factors. First, the detection anchor boxes on YOLACT \cite{Bolya_2019_ICCV} are more dense than YOLO v7 \cite{2022YOLOv7}, resulting in less frequent scenarios where $\mathcal{L}_{MPDIoU}$ has an advantage over $\mathcal{L}_{IoU}$ such as nonoverlapping bounding boxes. Second, the existing loss functions for bounding box regression have been improved during the past few years, which means the accuracy improvement is very limit, but there are still large room for the efficiency improvement.

We also compared the trend of bbox loss and AP value during the training period of YOLACT with different regression loss functions. As Figure \ref{COCO_Train}(a),(b) shows, training with $\mathcal{L}_{MPDIoU}$ performs better than most of the existing loss functions, i.e. $\mathcal{L}_{GIoU}$, $\mathcal{L}_{DIoU}$, which achieve higher accuracy and faster convergence. Although the bbox loss and AP value show great fluctuation, our proposed $\mathcal{L}_{MPDIoU}$ performs better at the end of training.

In order to better reveal the performance of different loss functions for bounding box regression of instance segmentation, we provide some of the visualization results as Figure \ref{resultVOC} and \ref{resultCOCO} shows. As we can see, we provide the instance segmentation results with less redudancy and higher accuracy based on $\mathcal{L}_{MPDIoU}$ other than $\mathcal{L}_{GIoU}$, $\mathcal{L}_{DIoU}$, $\mathcal{L}_{CIoU}$ and $\mathcal{L}_{EIoU}$. 

% \begin{table*}[ht]
%   \centering
%   \begin{tabular}{c|cc|cccccc|cccccc}
%   \toprule
%   Loss&FPS&Time&AP&$AP_{50}$&$AP_{75}$&$AP_{S}$&$AP_{M}$&$AP_{L}$&$AR_{1}$&$AR_{10}$&$AR_{100}$&$AR_{S}$&$AR_{M}$&$AR_{L}$\\

%   \midrule
%     $\mathcal{L}_{GIoU}$&78.5
%     &86.1&&&&56.1
%     \\
    
%     $\mathcal{L}_{DIoU}$&79.5
%     &86.1&&&&56.2\\
    
%     $\mathcal{L}_{CIoU}$&79.9
%     &86.1&&&&56.2\\
   
%     $\mathcal{L}_{EIoU}$&79.9
%     &86.1&&&&56.3\\
    
%     $\mathcal{L}_{MPDIoU}$&\textbf{81.1}
%     &86.1&&&&\textbf{56.3}\\
    
%     \bottomrule
%   \end{tabular}

%   \caption{Instance segmentation results of YOLACT \cite{Bolya_2019_ICCV}. The models are retrained using $\mathcal{L}_{GIoU}$, $\mathcal{L}_{DIoU}$, $\mathcal{L}_{CIoU}$ and $\mathcal{L}_{EIoU}$ by us, and the results are reported on PASCAL VOC 2007\&2012 validation set \cite{2015The}.}
% \end{table*}

\section{Conclusion}
In this paper, we introduced a new metric named $MPDIoU$ based on minimum points distance for comparing any two arbitrary bounding boxes. We proved that this new metric has all of the appealing properties which existing $IoU$-based metrics have while simplifing its calculation. It will be a better choice in all performance measures in 2D/3D vision tasks relying on the $IoU$ metric.

We also proposed a loss function called $\mathcal{L}_{MPDIoU}$ for bounding box regression. We improved their performance on popular object detection, scene text spotting and instance segmentation benchmarks such as PASCAL VOC, MS COCO, MTHv2 and IIIT5K using both the commonly used performance measures and also our proposed $MPDIoU$ by applying it into the state-of-the-art object detection and instance segmentation algorithms. Since the optimal loss for a metric is the metric itself, our $MPDIoU$ loss can be used as the optimal bounding box regression loss in all applications which require 2D bounding box regression.

As for future work, we would like to conduct further experiments on some downstream tasks based on object detection and instance segmentation, including scene text spotting, person re-identification and so on. With the above experiments, we can further verify the generalization ability of our proposed loss functions.



%% References


\bibliographystyle{elsarticle-num}
\bibliography{reference}






% \end{thebibliography}

\end{document}